\documentclass[11pt]{article}

\usepackage[utf8]{inputenc}
\usepackage{soul}
\usepackage{makecell}
\RequirePackage{amsthm,amsmath,amsfonts,amssymb}
\RequirePackage[colorlinks,citecolor=blue,urlcolor=blue]{hyperref}
\RequirePackage{graphicx,subcaption,makecell,float}
\RequirePackage{xcolor, xspace}
\RequirePackage{bbm,mathtools,enumitem}
\usepackage[top=1in, bottom=1in, left=1in, right=1in]{geometry}
\usepackage{bm}
\usepackage{microtype}
\usepackage{graphicx, subcaption}
\usepackage{booktabs} 
\usepackage{etoolbox}
\usepackage{multirow}

\usepackage{xcolor}
\usepackage{bbm}

\usepackage{amsmath}
\usepackage{breqn}

\usepackage{algorithmic}
\usepackage{algorithm}

\usepackage{lipsum}

\usepackage{caption}
\usepackage{subcaption}

\usepackage[backend=biber,
style=alphabetic
]{biblatex}
\newtheorem{example}{Example}
\newtheorem{theorem}{Theorem}
\newtheorem{lemma}{Lemma}
\newtheorem{prop}{Proposition}
\newtheorem{assumption}{Assumption}
\newtheorem{definition}{Definition}
\newtheorem{corollary}{Corollary}

\newcommand{\Exp}{\mathbb{E}}

\newcommand{\Hilb}{\mathcal{H}}
\newcommand{\A}{\mathcal{A}}
\newcommand{\State}{\mathcal{S}}

\newcommand{\Bee}{V}
\newcommand{\R}{\mathbb{R}}
\newcommand{\Prob}{\mathbb{P}}

\newcommand{\coeff}{\beta}

\newcommand{\holde}{\beta_0}
\newcommand{\maxholde}{\beta_{max}}
\newcommand{\holder}{\beta_0}
\newcommand{\disc}{\gamma}
\newcommand{\grad}{\tilde{\bm{g}}}

\allowdisplaybreaks

\newcommand{\p}{\phi}

\newcommand\inner[2]{\left\langle #1, #2 \right\rangle}

\newcommand{\norm}[1]{\left\lVert#1\right\rVert}
\newcommand{\bracket}[1]{\left(#1\right)}
\newcommand{\sq}[1]{\left[#1\right]}
\newcommand{\abs}[1]{\left|#1\right|}

\usepackage{setspace}
\title{Convergence and Optimality of Policy Gradient Methods in Weakly Smooth Settings}

 \author{
  Matthew S. Zhang
  \thanks{
  Department of Computer Science at
  the University of Toronto, and Vector Institute, \texttt{matthew.zhang@mail.utoronto.ca}
  }
  \and Murat A. Erdogdu\thanks{
  Department of Computer Science and Department of Statistical Sciences at
  the University of Toronto, and Vector Institute, \texttt{erdogdu@cs.toronto.edu}
 }
 \and Animesh Garg\thanks{
  Department of Computer Science at
  the University of Toronto, and Vector Institute, \texttt{garg@cs.toronto.edu}
 }
}
\begin{document}

\maketitle

\begin{abstract}
Policy gradient methods have been frequently applied to problems in control and reinforcement learning with great success, yet existing convergence analysis still relies on non-intuitive, impractical and often opaque conditions. In particular, existing rates are achieved in limited settings, under strict regularity conditions. In this work, we establish explicit convergence rates of policy gradient methods, extending the convergence regime to weakly smooth policy classes with $L_2$ integrable gradient. We provide intuitive examples to illustrate the insight behind these new conditions. Notably, our analysis also shows that convergence rates are achievable for both the standard policy gradient and the natural policy gradient algorithms under these assumptions. Lastly we provide performance guarantees for the converged policies.
\end{abstract}

\section{Introduction}

Modern Reinforcement Learning (RL) has solved challenges in diverse fields such as finance, healthcare, and robotics \cite{deng2016deep, yu2019reinforcement, kober2013reinforcement}. Nonetheless, the theory behind these methods remains poorly understood, with convergence results being limited to narrow classes of problems. Classical approaches to RL theory focus on tabular problems where discrete techniques can be applied (see \cite{agarwal2020optimality, sidford2018near}). However, most practical problems exist in continuous, high-dimensional domains \cite{doya2000reinforcement}, and may even be infinite-dimensional or non-compact.

Theoretical results in continuous domains do not effectively characterize practical algorithms. Value-based estimators have obtained strong results in some regimes such as linear MDPs, both in on- and off-line settings \cite{cai2019provably, yang2019sample}. In contrast to value-based methods, direct policy estimators possess numerous advantages, in that they are theoretically insensitive to perturbations in the problem parameters, and typically satisfy better smoothness assumptions. Nonetheless, bounds for direct parameterizations of the policy have been less successful. They either restrict the cardinality or size of the space \cite{agarwal2020optimality}, or apply strong assumptions on the policy and MDP \cite{liu2019neural, xu2020improving, liu2020improved}. This conflicts with practical results, where convergence often occurs without boundedness or smoothness preconditions on the function approximator.

Consequently, in this paper, we analyse two key questions: (i) how can we \textit{relax existing conditions on MDPs} while retaining guarantees for convergence, (ii) how can we bound \textit{the performance of the policies} under these conditions. Arguably, the convergence of gradient algorithms needs to rely on some constraints of the function class. Prior work has relied on assumptions of (a) smoothness and (b) absolute boundedness of the gradient. However, these conditions are overly restrictive and exclude many useful function approximators.

\textbf{Summary of Contributions.} We make contributions with respect to each of these assumptions: (a) strong smoothness is relaxed to weak smoothness (Hölder conditions); (b) absolute boundedness of the gradient is relaxed to $L_2$ integrability under the visitation distribution. While this is an important theoretical development, it also expands the scope of practical convergence results. We include many practical examples of MDPs and policies that satisfy our criteria, with applications to exploration and safety in reinforcement learning. To the best of our knowledge, ours is the first study to consider this setting.

Under these assumptions, we find (Corollary \ref{cor:convergence}) that with an appropriate learning rate, the gradients satisfy $\frac{1}{T} \sum_{t=1}^T \Exp \sq{\norm{\nabla J(\theta_{t-1})}^{2}} \leq \epsilon$ for both the standard and natural policy gradient for large enough $T, B$, where $T$ is the number of iterations and $B$ is the batch size. 
We also show (Corollary \ref{cor:opt}) that natural policy gradient satisfies the following bound:
\begin{align}
\hspace{-0.12in} \min_{t=0,1,\ldots {T-1}} J(\theta_*) - \Exp\sq{J(\theta_t)} \leq \epsilon + \mathcal{O} \left(\frac{\sqrt{E_{\Pi}} }{1-\gamma} \right),    
\end{align}
for large enough $T, B$,
where $E_\Pi$ can be tuned by choosing an appropriately regular policy class and $\theta_*$ is the maximizer of $J$. $E_\Pi$ is formally defined in \S5. Under a strong additional assumption, standard policy gradient also satisfies $\min_{t = 0,1 \ldots {T-1}} J(\theta_*) - \Exp \sq{J(\theta_t)} \leq \epsilon$ for large enough $T, B$.
In the strictly smooth limit these results have previously been discovered \cite{agarwal2020pc, xu2020improving, zou2019finite}, although our results hold for a wider range of functions and MDPs.

The remainder of the paper is structured as follows: in \S2 we cover the mathematical formulation of MDPs; in \S3 we introduce the policy gradient algorithm as well as our assumptions. In \S4, we list several candidate policies that satisfy our assumptions, and demonstrate their utility in a variety of contexts. \S5 then states our main results; \S6 summarizes works related to optimization and RL theory.

\section{Background}
\subsection{Markov Decision Processes}
Let a state-space be denoted by $\mathcal{S}$, and an action-space by $\mathcal{A}$. Let a transition measure $P(\cdot|s,a)$ and a reward measure $R(\cdot|s,a)$ be probability measures on $\mathcal{S}$ and $\R$ respectively, both conditioned on variables $(s,a) \in \mathcal{S} \times \mathcal{A}$. A Markov Decision Process $\mathcal{M}$ is formally defined as a tuple $(\mathcal{S}, \mathcal{A}, P, R, \gamma)$, where $\gamma \in [0,1)$ is the discount factor. By abuse of notation, we use the same notation for a measure and its density, unless otherwise specified. Let $\norm{z} = \norm{z}_2$ the $2$-norm for vectors, $\norm{z}_{op}$ the operator norm for matrices, and $\norm{p-q}_{TV} \triangleq \int \abs{p(x)-q(x)} dx$ the total variation distance for signed measures. Hereafter we assume that the absolute magnitude of the rewards are bounded, i.e. $R(\cdot|s,a)$ only has support on $[-\alpha, \alpha]$ for some $\alpha \geq 0$, and all $s,a$.

\textbf{Policies:} For a given state $s \in \mathcal{S}$, we denote a stochastic policy with $\pi(\cdot|s)$, which is a probability distribution over $\mathcal{A}$. 

\textbf{Trajectories:} To generate trajectories, we start from an initial state distribution $\rho$, and then at each time $t \in \mathbb{N}$, we sample an action from the policy: $a_t \sim \pi(\cdot | s_t)$. Subsequently a state and reward are queried as $s_{t+1} \sim P(\cdot|a_t, s_t), r_t \sim R(\cdot|a_t, s_t)$, and the process continues indefinitely. Consequently $\pi, \rho, \mathcal{M}$ together parameterize a probability distribution on the set of trajectories. Letting $\rho$ be fixed, we write this as $\{(s_t, a_t), t = 0, 1, 2, \ldots\} \sim \text{MDP}$ following $\pi$.

\textbf{Value Functions:} We can define the value function as: $V_\pi(s) \triangleq$ $\Exp[\sum_{t=0}^{\infty} \gamma^t r_t|s_0=s],$
and the Q-function as: $Q_\pi(s, a) \triangleq \Exp[\sum_{t=0}^\infty \gamma^t r_t|s_0=s, a_0=a]$.
Note that both expectations are taken over trajectories following the policy $\pi$. If $\abs{r_t} \leq \alpha$ almost surely, both functions are bounded by $[-\alpha/(1-\gamma),\alpha/(1-\gamma)]$ almost surely as well. We can also define the advantage function $A_\pi(s,a) \triangleq Q_\pi(s,a) - \Bee_\pi(s).$ 

\textbf{Discounted Visitation:}It will be useful to define the sum of time-discounted visitation probability densities through the following: $d_{\pi}^{s}(s,a) = (1-\gamma) \sum_{t=0}^{\infty} \gamma^t p_t(s,a|s_0=s)$ where $p_t(\cdot|s_0=s)$ is the conditional probability density of $s, a$ being sampled at time $t$ from the MDP following $\pi$, given the initial state $s_0=s$. We overload notation and write $d_\pi^\rho(s,a) = \int d_\pi^{s'}(s,a) \rho(s') \, ds'.$ 
This defines a probability density function on $\State \times \A$. We also write $H_{\theta}^{\rho}(s) = \int d_\theta^\rho(s,a) \, da$ for the state-component of the visitation distribution.

\textbf{Reinforcement Learning}
A reinforcement learning agent is one which produces a sequence of policies $\pi_t$ based on the queried states $s_t, r_t$, which seeks to iteratively maximize the value function: $J(\pi) = \int V_\pi(s) \rho(s) \, ds$. The existence of an optimum in the space of stochastic functions has been shown as a classical result \cite{bellman1954theory}.

\section{Algorithms}
\subsection{Policy Class}

In this work, we limit our discussion to exponential policy classes which are continuously differentiable. In particular, we denote the distribution of an exponential policy, parameterized by a variable $\theta \in \Theta \subseteq \R^N$, such that $\pi_{\nu_\theta}(a|s) = \frac{\exp(\nu_\theta(s,a))}{\int_{\A} \exp(\nu_\theta(s,a)) \, da}$ where $\nu_\theta: \State \times \A \mapsto \R$.
We require that the integral $\int_{\A} \exp(\nu_\theta(s, a)) \, da < \infty$ is finite for all $\theta \in \Theta, s \in \State$, and that $\nu_\theta(s,a)$ is differentiable in $\theta$ for all $s, a$. Let us define $\pi_\theta \triangleq \pi_{\nu_\theta}$, $J(\theta) \triangleq J(\pi_\theta)$ and use $\theta$ instead of $\pi_\theta$ in subscripts where there is no confusion. Let us denote the score function as $\psi_{\theta}(s,a) = \nabla_\theta \log \pi_\theta(a|s)$. Then the gradient can be written as $\nabla J(\theta) = \Exp_{s,a \sim d_{\theta}^\rho}[Q_{\theta}(s,a) \psi_{\theta}(s,a)]$. 
While successful tabular approaches rely on explicit computation of each softmax probability, this is not feasible for most MDPs where the action space is infinite and possibly uncountable. Typically some form of well-chosen function class is required to address this issue. In this work, we consider all softmax functions that satisfy the following smoothness properties:
\begin{assumption}
    (Smoothness of Policy Class) Consider policies $\pi_\theta \propto \exp(\nu_\theta)$. We require that $\pi$ obeys the following two smoothness conditions:
    \begin{align}
        &\hspace{-0.13in}\int_{\A} \pi_{\theta}(a|s) \log \frac{\pi_{\theta}(a|s)}{\pi_{\theta+\eta}(a|s)} \, da \, \leq C_{\nu,1} \norm{\eta}^{\beta_1} \label{eq:KL}, \\
        \label{eq:TV}
        &\hspace{-0.13in} \norm{\psi_{\theta}(s,a) - \psi_{\theta+\eta}(s,a)} \!\leq\! C_{\nu,2} \norm{\eta}^{\beta_2} \!\!\!\!,
    \end{align}
    where the constants $C_{\nu,1}, C_{\nu,2} \geq 0$, $\beta_1 \in [1, 2], \beta_2 \in (0,1]$ are valid for all $\theta, s,a$. Consequently we define $\holde = \min(\beta_1/4, \beta_2)$ as the dominant order of smoothness. 
    \label{as:smooth}
\end{assumption}
    It will also be useful in our analysis to define $\maxholde = \max(\beta_1/4, \beta_2)$.
    We note that \eqref{eq:KL} is a Hölder condition on the Kullback–Leibler (KL) divergence of the policies, while \eqref{eq:TV} is a Hölder requirement on the score function.
    
    \textbf{Remarks:} $\beta_1 < 2, \beta_2 < 1$ are weakly smooth cases. This is a weaker assumption than traditional assumptions on Lipschitz smoothness; particularly, it allows for slow tail decay. 
    It is also possible to relax this assumption to local conditions (i.e. only holding when $\norm{\eta} \leq C$) with Lipschitz behaviour at larger scales. 

    We introduce an additional assumption on the second moment of the score function:
    \begin{assumption}
       (Boundedness of Moments) Assume that the score function is absolutely bounded in $L_2$ across all policies i.e. that the following holds for all $\theta$
       \begin{align}
           &\int_{\State} \int_{\A} \norm{\psi_{\theta}(s,a)}^2 d_{\theta}^\rho(s,a) da \, ds \leq \psi_{\infty},
       \end{align}
       for any $\theta$ in our parameter space, where $\psi_{\infty} < \infty$ is a constant independent of $\theta$.
       \label{as:moment}
       
    \end{assumption}

\textbf{Remarks:} Higher order integrability assumptions are possible. In fact, if $\norm{\psi_\theta} \leq \sqrt{\psi_\infty}$ holds $d_\theta^\rho$-almost surely, we recover the standard bounded gradient assumption found in other works, e.g. \cite{xu2020improving, liu2020improved}.

   Finally, we require the following standard assumption (see e.g. \cite{xu2020improving, zou2019finite}) which we use to show smoothness of the objective function. 
    \begin{assumption}
        \label{as:erg}
        (Ergodicity) We have for all states $s_0 \in \State$:
        \begin{equation*}
            \norm{\Prob^n_{\theta}(\cdot|s_0) - \rho_*(\cdot)}_{TV} \leq C_0 \delta^n,
        \end{equation*}
        where $\Prob^n_{\theta}$ is the $n$-step state transition kernel following $\pi_\theta$, $\rho_*$ is the invariant state distribution, $C_0 \geq 0, \delta < 1$ are constants independent of $s_0, \theta$.
    \end{assumption}

\subsection{Policy Gradient}
\begin{algorithm}[!th]
\begin{algorithmic}[1]
    \STATE Initial parameter $\theta_0$.
    \FOR{Step $t = 1, \ldots, T$}
       \FOR{$i=1, \ldots B$}
            \STATE Let $j \sim \text{Geom}(1-\gamma)$, $h \sim \text{Geom}(1-\gamma^{1/2})$, $\tau = j + h$.
            \STATE Sample $\left(s_{0}, a_{0}, \ldots s_{\tau}, a_{\tau} \right) \sim \text{MDP}$ following $\pi_{\theta_{t-1}}$.
            \STATE $s_{t,i} \gets s_{j}, a_{t,i} \gets a_{j}$.
            \STATE $v_{t,i} \gets \sum_{u=j}^\tau \gamma^{(u-j)/2} r_{u}$, $r_{u} \sim R(\cdot|s_{u}, a_{u})$. 
        \ENDFOR
        \STATE Choose $h_t$ specified in our learning rates section.
        \STATE  $\theta_{t} \gets \theta_{t-1} + \frac{h_t }{B} \sum_{i=1}^B v_{t,i} \psi_{\theta_{t-1}}(s_{t,i}, a_{t,i})$.
    \ENDFOR
    \STATE Return $\theta_T$.
\end{algorithmic}
\caption{Policy Gradient for Hölder Smooth Objectives}
\end{algorithm}

\begin{algorithm}[!th]
\begin{algorithmic}[1]
    \STATE Initial parameter $\theta_0$, stability parameter $\xi \in (0,1]$.
    \FOR{Step $t = 1, \ldots, T$}
        \FOR{$i=1, \ldots B$}
            \STATE Let $j \sim \text{Geom}(1-\gamma)$, $h \sim \text{Geom}(1-\gamma^{1/2})$, $\tau = j + h$.
            \STATE Sample $\left(s_{0}, a_{0}, \ldots s_{\tau}, a_{\tau} \right) \sim \text{MDP}$ following $\pi_{\theta_{t-1}}$.
            \STATE $s_{t,i} \gets s_{j}, a_{t,i} \gets a_{j}$.
            \STATE $v_{t,i} \gets \sum_{u=j}^\tau \gamma^{(u-j)/2} r_{u}$, $r_{u} \sim R(\cdot | s_{u}, a_{u})$.
        \ENDFOR        
        \FOR{$i=1, \ldots B$}
            \STATE Let $j \sim \text{Geom}(1-\gamma)$.
            \STATE Sample $\left(s'_{0}, a'_{0}, \ldots s'_{j}, a'_{j} \right) \sim \text{MDP}$ following $\pi_{\theta_{t-1}}$.
            \STATE $s'_{t,i} \gets s'_{j}, a'_{t,i} \gets a'_{j}$.
        \ENDFOR
        \STATE Choose $h_t$ specified in our learning rates section.
        \STATE $K_t \gets \frac{1}{B} \sum_{i=1}^B \psi_{\theta_{t-1}}(s'_{t,i}, a'_{t,i}) \psi_{\theta_{t-1}}^\top(s'_{t,i}, a'_{t,i})$.
        \STATE $\theta_{t} \gets \theta_{t-1} +  \frac{h_t }{B} (K_t + \xi I)^{-1}  \sum_{i=1}^B v_{t,i} \psi_{\theta_{t-1}}(s_{t,i}, a_{t,i})$.
    \ENDFOR
    \STATE Return $\theta_T$.
\end{algorithmic}
\caption{Natural Policy Gradient for Hölder Smooth Objectives}
\end{algorithm}

Given these assumptions on the policy class, we can apply direct policy ascent on the space of parameters in order to get the gradient update
\begin{equation}
    \theta_t = \theta_{t-1} + h_t \nabla_{\theta} J(\theta_{t-1}),
\end{equation}
where $h_t \in \R$ is an adaptive step size. Alternatively, natural policy gradient (NPG), first introduced by \cite{kakade2001natural}, is a parameter invariant method that applies the following update
\begin{equation}
    \theta_t = \theta_{t-1} + h_t K^{\dagger}(\theta_{t-1}) \nabla_{\theta} J(\theta_{t-1}),
\end{equation}
    where $K(\theta) = \Exp_{s,a \sim d_{\theta}^\rho} \sq{\psi_\theta(s,a) \psi_{\theta}(s,a)^\top}$. Here $(\cdot)^{\dagger}$ is the matrix pseudo-inverse. The advantage of this method is that the optimization landscape becomes well-behaved.
    
    Since the true loss function and Fisher information matrix are not available to us, we estimate each of them through sampling. In particular, we use the following minibatch estimators for $\nabla J$ and $K$:\small
    \begin{align}
        &\widehat{\nabla J(\theta_{t-1})} = \frac{1}{B} \sum_{i=1}^B v_{t,i} \psi_{\theta_{t-1}}(s_{t,i}, a_{t,i}), \\
        &K_t = \frac{1}{B} \sum_{i=1}^B \psi_{\theta_{t-1}}(s_{t,i}, a_{t,i}) \psi_{\theta_{t-1}}^\top(s_{t,i}, a_{t,i}),
    \end{align}\normalsize
    and we use $(K_t + \xi I)^{-1}$ to approximate the inverse of the Fisher matrix, where $\xi > 0$ is a hyperparameter that guarantees the estimator is numerically stable, $v_{t,i}$ is an unbiased estimator for $Q_{\theta_{t-1}}(s,a)$ given in Algorithms 1-2 wherein we follow \cite[Algorithm 1]{zhang2020global}, and to sample from the occupancy measure $d_\theta^\rho$, we sample trajectories following \cite[Algorithm 1]{agarwal2020optimality}. This procedure is summarized in our Algorithms.
    
    \subsection{Learning Rates}
    In the sequel, we consider the following learning rates: \textbf{(i)} constant $h_t = \lambda$, \textbf{(ii)} dependent on the total number of steps $h_t = \lambda T^{\frac{\holde-1}{\holde+1}}$, \textbf{(iii)} decaying $h_t = \lambda t^{-q}, q \in [0,1)$. We also state our Theorems \ref{th:con}-\ref{th:opt} more generally for any step size sequence $h_t$. 

\section{Applications}

    \begin{figure}
    \centering
        \begin{subfigure}{0.3\linewidth}
            \centering
            \includegraphics[width=\textwidth]{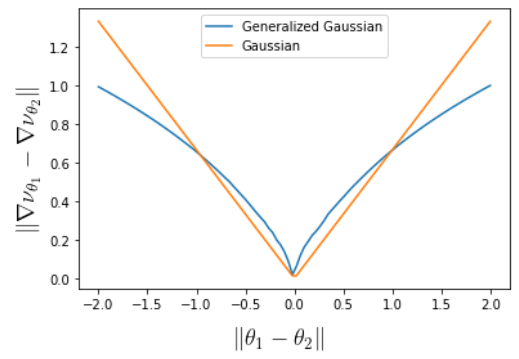}
            \label{fig:potential}
         \end{subfigure}
        \begin{subfigure}{0.3\linewidth}
            \centering
            \includegraphics[width=\textwidth]{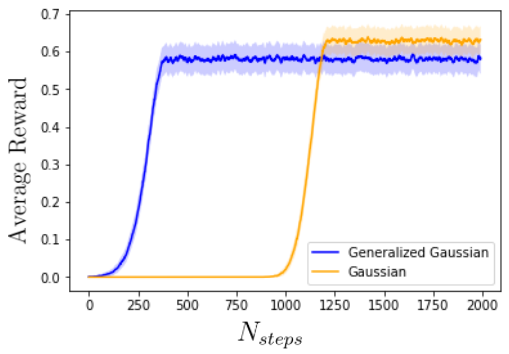}
            \label{fig:explore}
        \end{subfigure} 
        \begin{subfigure}{0.3\linewidth}
            \centering
            \includegraphics[width=\textwidth]{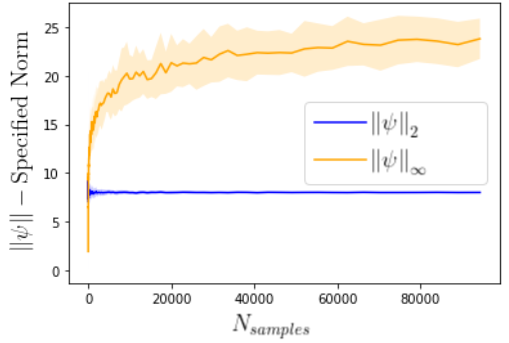}
            \label{fig:l2}
        \end{subfigure}
    \caption{(a) \textbf{Tail Growth}: Comparing the growth of $\psi_\theta$ in one-dimension for Gaussian policies versus the Generalized Gaussian (Example \ref{ex:ggp}) with $\alpha = 0.5$, for the $[0,0]$ state in the MountainCar environment. (b) \textbf{Exploration Performance}: Comparing the performance of Generalized Gaussian and the standard Gaussian policy, with $\alpha=0.5$, for the reward function found in Equation \eqref{eq:example}, $\abs{\theta^* - \theta} = 3.9$. The Generalized Gaussian significantly outperforms during the exploration phase. The result is similar for both PG and NPG. (c) \textbf{Gradient Norm Growth}: Comparing the growth of Example \ref{ex:l2} using as estimate of the $L_2$ norm described by Assumption \ref{as:moment}, versus $\max_{n} \norm{\psi_{\theta}(s_n,a_n)}$ with growing number of samples. While our criterion is stable, the $\max$ diverges logarithmically.}
    \label{fig:generalized_gaussian}
    \end{figure}
    
    We note two prominant applications of our assumptions: (i) an application of Assumption \ref{as:smooth} to exploration has been explicitly shown in \cite{chou2017improving}, (ii) Assumption \ref{as:moment} has been shown to apply to Safe RL via the work of \cite{papini2019smoothing}. Some additional examples will serve to illustrate these points below.
    
    For ease of demonstration, we consider policies and environments which independently satisfy Assumptions \ref{as:smooth}-\ref{as:moment} and Assumption \ref{as:erg} respectively, so long as the other component is sufficiently regular. The following policies illustrate why we might value weak smoothness:
    \begin{example}
        (Generalized Gaussian Policy) If we choose the parameter $\kappa \in (1,2]$, we can choose the generalized Gaussian distribution to parameterize our policy:
        \begin{equation}
            \nu_\theta(a|s) = -\abs{\inner{\p(s,a)}{\theta}}^{\kappa}.
        \end{equation}
        See Figure \ref{fig:generalized_gaussian}(a) for a visualization of the smoothness of this policy.
        \label{ex:ggp}
    \end{example}
    This distribution is covered by our assumptions; in contrast, previous works only permitted the strictly Gaussian distribution, where $\kappa = 2$. In particular, the tails of this distribution decay much more slowly than the tails of the Gaussian distribution, which has applications to exploration-based strategies. Indeed, let us consider the following single-state exploration problem with the following (deterministic) reward
    \begin{align}
        r(a_t) = \left(1 - (a_t - \theta^*)^2 \right) \mathbbm{1}_{\abs{a_t - \theta^*} \leq 1},
        \label{eq:example}
    \end{align}
     with policies $\nu_\theta(a) = -\abs{a - \theta}^\kappa$ for $\kappa = 2$ (a Gaussian policy) and $\kappa \in (1,2]$ (a generalized Gaussian). $\theta^* \in \R$ is an unknown target. If $\theta^*$ is far from our initial parameter, the agent will receive no gradient information so long as it does not sample actions from the region of interest $[\theta^*-1, \theta^*+1]$. For a policy with exponent $\kappa$, this occurs with probability
    \begin{align*}
        &\pi_{\kappa,\theta_0}(a_t \in [\theta^*-1, \theta^*+1]) \\
        &\quad =\frac{1}{2\Gamma(\kappa+1/\kappa)} \int_{\theta^*-1}^{\theta^*+1} \exp (-\abs{a-\theta_0}^\kappa) da,
    \end{align*}
    where $\pi_{\kappa,\theta_0}$ is the policy measure under the generalized Gaussian with exponent $\kappa$ and parameter $\theta_0$.
    If $\mathcal{U} = [\theta^* - 1, \theta^* + 1]$
    \begin{align*}
        &\pi_{\kappa,\theta_0}(a_t \in \mathcal{U}) - \pi_{2,\theta_0}(a_t \in \mathcal{U}) \\
        &\quad \geq \frac{1}{2\Gamma(\kappa+1/\kappa)} \int_{\theta^*-1}^{\theta^*+1} \exp (-\abs{a-\theta_0}^\kappa) \\
        &\ \ \ \ - \exp (-\abs{a-\theta_0}^2 + \log 2) da,
    \end{align*}
    which is $>0$ by simply comparing the terms in the exponents, when $\kappa \ll 2$ and $\abs{\theta^* - \theta_0} \gg 0$. This difference in probability can improve sample efficiency by many orders of magnitude. The empirical performance of the two policies is found in Figure \ref{fig:generalized_gaussian}(b), with a large improvement in number of samples needed to discover the correct action. This example can be easily generalized to more complex bandits/MDPs.

    Another example shows the richness of the weakly smooth assumption:
    \begin{example}
        ($p$-Harmonic minimizers) It is known \cite{coscia1999holder, lindqvist2017notes} that local minimizers $\nu$ to the $p$-Harmonic functional, for $p(x): \mathbb{R}^d \mapsto \mathbb{R}$
        \begin{equation}
            \mathcal{F}(\nu) \triangleq \int_{\mathbb{R}^d} \norm{\nabla \nu}^{p(x)} \, dx,
        \end{equation}
        are weakly smooth of some order $L(p) < 1$ when $p(x) > 1$. 
    \end{example}
    One can also restrict the integration above to a compact set. Consequently, these can serve as interesting potential functions. Note that we can add any bounded potential with  bounded and Lipschitz gradient to such functions while preserving Hölder regularity. Weak smoothness has also been shown for many other elliptic families of PDEs \cite{hoeg2020regularity,sciunzi2014regularity}, which may also motivate some candidate policies. 
    
    To illustrate the distinction of Assumption \ref{as:moment} from standard $\norm{\cdot}_\infty$ bounds, consider the following policy class:
    \begin{example}
     \label{ex:l2}
    (Safe Policies) Consider the following potential for $\theta \in [-1,1], \norm{\p^*} \leq 1$:
        \begin{equation}
            \nu_\theta(s,a) = -\theta \log \norm{\p(s,a) - \p^*}.
        \end{equation}
    \end{example}
    Under uniform dynamics and a uniform distribution of $\p(s,a)$ on a ball of radius $1$ around the origin, this family satisfies Assumption \ref{as:moment}, but not the standard assumption of absolute boundedness $\sup_{s,a} \norm{\psi_\theta(s,a)}_\infty < \infty$ (see Figure \ref{fig:generalized_gaussian}(c) for an empirical demonstration). This policy explicits avoids the state-action region around $\p^*$; this can arise practically when considering safety or instability constraints in RL.  

    \section{Main Results}
    In the sequel, define $E_{\Pi} \hspace{-3pt}= \max_{\theta \in \Theta} \hspace{-3pt} \Exp_{s,a \sim {d_{\theta}^\rho}} \hspace{-3pt} \sq{\norm{\psi_{\theta}(s,a)^\top K(\theta)^\dagger \nabla J(\theta) \hspace{-3pt} - \hspace{-3pt} A_{\theta}(s,a)}^2}$ and the quantity $D_\infty = 1 + \sup_{\theta_1, \theta_2 \in \Theta} \norm{\frac{ d_{\theta_1}^\rho}{d_{\theta_2}^\rho}}_\infty$.
    
    \begin{theorem}
        \label{th:con}
        Under Assumptions \ref{as:smooth}-\ref{as:erg}, \textbf{Policy Gradient} and \textbf{Natural Policy Gradient} achieves the following convergence:
        \small
    \begin{align*}
        \sum_{t=1}^{T} h_t \Exp \sq{\norm{\nabla J(\theta_t)}^2} &\leq C_{k,1} \left(J_* - J(\theta_{0})\right) + \frac{ C_{k,2}}{1-\gamma} \sum_{t=1}^T  \Biggl(h_t^{\frac{\beta_1}{4} +1}\left(\Exp \sq{\norm{\nabla J(\theta_t)}^{\frac{\beta_1}{4}+1}} + \left(\frac{\sigma}{(1-\gamma)\sqrt{B}} \right)^{\frac{\beta_1}{4}+1}\right) \\
        &\qquad + h_t^{\beta_2 + 1}\left(\Exp\sq{\norm{\nabla J(\theta_t)}^{\beta_2 + 1}} + \left(\frac{\sigma}{(1-\gamma)\sqrt{B}} \right)^{\beta_2 + 1} \right) \Biggr),
    \end{align*}
    \normalsize
        where the $k$ in $C_{k,1}, C_{k,2}$ is either $PG, NPG$ depending on whether policy gradient or natural policy gradient is considered. The exact values of these constants are defined in the Appendix. They do not depend on $\epsilon, \gamma$ for either algorithm. $\sigma = 3 \alpha \sqrt{\psi_\infty}$ controls the variance of the gradient. $J(\theta_0)$ is the initial performance and $J_* = \sup_{\theta \in \Theta} J(\theta)$, which is finite due to the boundedness of the reward. $B$ is the batch size and the remaining constants are specified in the Appendices.
    \end{theorem}
    
    \textbf{Remarks:} If we replace Assumption \ref{as:moment} with an almost sure bound on $\norm{\psi_\theta}$, the exponent $\beta_1/4 + 1$ becomes instead $\beta_1/2 + 1$. 
    
    With respect to the ergodicity mixing rate $\delta$, $C_{k,2}$ scales as $1/(1-\delta)$ for both algorithms, which is analogous to other works with ergodicity \cite[Proposition 1]{xu2020improving}.
    
    \begin{corollary}
    \label{cor:convergence}
        (Rates under various step-size schemes) Table \ref{tab:convergence} encapsulates the orders of growth of $\frac{1}{T} \sum_{t=1}^T \Exp \sq{\norm{\nabla J(\theta_t)}^2}$ for each of the learning rates examined in our paper, for the choice of $\lambda$ sufficiently small and $B \gtrsim \frac{\sigma^2}{(1-\gamma)^2}.$
    \end{corollary}
    
    \begin{table*}[t!]
    \caption{Results for various learning rate schemes, for both policy gradient and natural policy gradient. We only track the primary dependence in $T, B, \gamma$. For the decaying learning rate, we define the coefficient $f(q, \holde) = \min(\frac{2q\holde}{1-\holde}, 1-q)$. In each case we require at least $\lambda^{\holde} \lesssim \frac{1-\gamma}{C}$ where $C$ does not depend on $\gamma, \epsilon$.}
    \centering
    \label{tab:convergence}
    \begin{tabular}{ccc}
        \toprule
        $h_t$ & Order & Considerations\\
        \midrule 
        $\lambda$ & $O(\lambda^{-1} T^{-1}) + O((\sqrt{B}(1-\gamma))^{-\holde-1}) + O((\lambda^{\holde}(1-\gamma))^{-1})^{\frac{2}{1-\holde}})$ & $\lambda$-dependent Bias \\
        \hline
        $\lambda T^{\frac{\holde-1}{\holde+1}}$ &$O(\lambda^{-1}T^{-\frac{2\holde}{1-\holde}}) + O(T^{\frac{\holde^2-\holde}{\holde+1}}(\sqrt{B}(1-\gamma))^{-\holde-1})$ \\ \hline
        $\lambda t^{-q}$ & $ \tilde{O}(\lambda^{-1} T^{-f(q,\holde)}) + O(T^{-q\holde} (\sqrt{B}(1-\gamma))^{-\holde-1})$ &  \\ 
        \bottomrule
    \end{tabular}
\end{table*}
    
    For global optimality, policy gradient requires another opaque assumption in order to demonstrate convergence:
    \begin{assumption}
        \label{as:glob}
        (Requirements for Policy Gradient) Assume that there is a $\theta_* \in \Theta$ where $J$ attains its maximum. Furthermore, let $\theta \in \Theta$ be any parameter. Then, we assume that $J$ is $m$-dominated for any $m > 0$, i.e. that the following holds
        \begin{align*}
            &J(\theta_*) - J(\theta) \leq \frac{m}{1-\gamma} \inner{\theta_* - \theta}{\nabla J(\theta)}.
        \end{align*}
        Furthermore, suppose that $\text{Diam}(\Theta) \triangleq \sup_{\theta_1, \theta_2 \in \Theta} \norm{\theta_1 - \theta_2} < \infty$.
    \end{assumption}
    See \cite[Lemma 3(a)]{bhandari2019global} for analogous conditions, which are often violated in practice. 
    
    \begin{theorem} 
        \label{th:opt}
        Let $\theta_* = \arg \max_{\theta \in \Theta} J(\theta)$. Under Assumptions \ref{as:smooth}-\ref{as:erg}, \textbf{Natural Policy Gradient} is bounded with the following for $t \in 1 \ldots T$:\small
    \begin{equation*}
        \begin{split}
            J(\theta_*) - \Exp \sq{J(\theta_{t-1})} &\leq \frac{C_{NPG,3}}{1-\gamma} h_t^{\beta_1-1} \left(\left(\frac{\sigma}{(1-\gamma) \sqrt{B}}\right)^{\beta_1}  + \Exp \sq{\norm{\nabla J(\theta_{t-1})}^{\beta_1}} \right) \\
            &\qquad + \frac{C_{NPG,4} h_t^{\beta_2}}{1-\gamma} \left(\left(\frac{\sigma}{(1-\gamma) \sqrt{B}} \right)^{\beta_2+1} + \Exp \sq{\norm{\nabla J(\theta_{t-1})}^{\beta_2+1}} \right) \\
            &\qquad + \frac{C_{NPG,5}}{1-\gamma} \bracket{\frac{\sigma}{(1-\gamma) \sqrt{B}} + \frac{\sqrt{E_\Pi}}{\sqrt{\psi_{\infty}}} + \Exp \sq{\norm{\nabla J(\theta_{t-1})}}}. 
        \end{split}
    \end{equation*}\normalsize Here, $E_{\Pi}$ is a policy dependent parameter that serves to lower bound the optimality of the function class, $\sigma$ is the variance from Theorem \ref{th:con}, and $\xi$ is the stability constant found in Algorithm 2.
    
    If, additionally, Assumption \ref{as:glob} is added, then the standard \textbf{Policy Gradient} is bounded by the following for $t \in 1 \ldots T$:
        \begin{equation}
            J(\theta_*) - \Exp{\sq{J(\theta_{t-1})}} \leq \frac{m\text{Diam}(\Theta)}{1-\gamma} \Exp \sq{\norm{\nabla J(\theta_{t-1})}},
        \end{equation}
    where $\text{Diam}(\Theta)$ is defined in Assumption \ref{as:glob}.
    \end{theorem}
    Here, $\theta_*$ is the minimizer from Assumption \ref{as:glob}, and $C_{NPG,3-5}$ are not dependent on $h_t, B, T, \gamma$ and are stated explicitly in the appendices.
    For natural policy gradient, there are no additional assumptions apart from the bias term $E_{\Pi}$ being finite; this is bounded under standard assumptions (see  \cite[Remark 6.4]{agarwal2020optimality}). This is a major advantage of NPG over its vanilla counterpart, which requires a strong additional regularity condition.
    
    For both natural and standard policy gradient, if we take the minimum over $t = 1 \ldots T$, we obtain the rates in the following corollary.
    
    \begin{corollary} \label{cor:opt}
        For $\lambda$ sufficiently small and the learning rate $h_t = \lambda T^{\frac{\holde-1}{\holde+1}}$, for \textbf{Policy Gradient} the following holds under Assumptions \ref{as:smooth}-\ref{as:glob}
        \begin{align*}
        \min_{t=0, \ldots T-1} J(\theta_*) - \Exp \sq{J(\theta_t)} \leq \epsilon.
        \end{align*}
        For \textbf{Natural Policy Gradient} the following holds under Assumptions \ref{as:smooth}-\ref{as:erg}
        \begin{align*}
        \min_{t=0, \ldots T-1} J(\theta_*) - \Exp \sq{J(\theta_t)} \leq \epsilon + \frac{\sqrt{D_{\infty}E_\Pi}}{1-\gamma}.
        \end{align*}
        Recall that $E_{\Pi}$ is an approximation error, and $D_\infty$ measures the irregularity of the initial distribution. For either algorithm if $h_t = \lambda T^{\frac{\holde-1}{\holde+1}}$,
        \begin{align*}
            T \geq \epsilon^{-\frac{\holde+1}{\holde}} (1-\gamma)^{-\frac{2\holde^2+3\holde+1}{2\holde^2}}, \qquad B \gtrsim \epsilon^{-2} (1-\gamma)^{-\frac{4\holde^2+5\holde-1}{\holde(\holde+1)}}.
        \end{align*}
        For both algorithms if instead $h_t = \lambda$,
                \begin{align*}
            T \geq \epsilon^{-\frac{\holde+1}{\holde}} (1-\gamma)^{-\frac{2+\holde}{\holde}}, \qquad B \gtrsim \epsilon^{-\frac{4}{\holde+1}} (1-\gamma)^{- \frac{6+2\holde}{\holde+1}}.
        \end{align*}
    \end{corollary}

\section{Related Work}

\subsection{Optimization and Stochastic Approximation}

We primarily refer to work on stochastic approximation, which began with the work by authors \cite{polyak1992acceleration, kushner2003stochastic}, who established basic conditions for convergence for linear approximation procedures, with rates being obtained under strong assumptions. Tighter bounds have recently been achieved through improved analysis and techniques, both in asymptotic and non-asmyptotic contexts \cite{chen2016statistical,lakshminarayanan2018linear,jain2018parallelizing}.

The theory for optimizing weakly smooth rather than Lipschitz functionals was primarily developed in the following works \cite{devolder2014first,nesterov2015universal, yashtini2016global}, introducing the definition of weak-smoothness through Hölder conditions, and showing convergence via smoothing or fast decaying learning rates. Lastly, our analysis relies heavily on the theory of ergodicity for MDPs. We build on the works of \cite{mitrophanov2005sensitivity} which yields perturbation bounds on the state distribution, and subsequent improvements in the assumptions and condition numbers \cite{ferre2013regular, rudolf2018perturbation, mao2020perturbation}.

\subsection{Reinforcement Learning}

The general formulation of reinforcement learning can be attributed to Bellman's formulation of Markov Decision processes \cite{bellman1954theory}. Gradient-based approaches were proposed to solve direct policy parameterizations \cite{williams1992simple}; developments in this classical setting include \cite{sutton1999between, konda2000actor, kakade2003sample}. These works established asymptotically tight bounds for convergence in the tabular setting, while outlining rough conditions for convergence when feature transformations were applied. The introduction of natural gradient techniques \cite{kakade2001natural}, which borrowed from similar work in standard optimization \cite{amari1998natural}, yielded improved convergence with respect to policy condition numbers. In particular, strong convergence holds for domains such as the linear quadratic regulator \cite{fazel2018global, tu2018least} and other linearized problems.

Even so, lower bounds for general problems can be quite pessimistic, especially when the conditions are ill-specified \cite{sutton2000policy}. This debate has attracted renewed focus in recent years, with an on-going discussion on the quality of representation and its effect on learnability \cite{du2019good, van2019comments}. Nonetheless, real world problems are either continuous or well-approximated by continuous algorithms, with smooth state-space. \cite{agarwal2020pc,agarwal2020optimality} provided a convergence and optimality result for both tabular and linear settings, but only when the action space was discrete and the problem was deterministic. Other results in this setting include \cite{mei2021leveraging, zhang2020variational, mei2020global, zhang2021convergence}. \cite{xu2020improving, kumar2019sample} focus on general settings, but only under generous smoothness and boundedness assumptions. Numerous works have since focused on feature representations in policy learning, particularly through use of neural networks \cite{thomas2017policy, wang2019neural, liu2019neural}; these apply similarly strict assumptions on the problem class in order to achieve good rates of convergence.

We would like to comment extensively on the results of \cite{liu2020improved}, which obtains highly competitive rates for PG and NPG, of $O(\epsilon^{-4})$ and $O(\epsilon^{-3})$ respectively. While our rate for NPG is worse at $\to O(\epsilon^{-4}), \holde \to 1$, this is because of numerous differences between our formulations. \cite{liu2020improved} rely on more complex sampling and natural gradient procedures, particularly requiring stochastic gradient descent in order to solve for the NPG update vector. It is unclear whether this technique can generalize to the weakly smooth regime. Instead, we analyze a much simpler algorithm that involves direct estimation of the Fisher information matrix, with an additional cost in $\epsilon$, while also handling non-constant learning rates.

Our results are simultaneously valid for continuous settings, while removing many of the strict assumptions found in previous results. In particular, smoothness of the policy class and boundedness of the gradient limited the scope of policies. We build upon work in weakly smooth optimization to relax these assumptions.
\section{Discussion}

In this work, we established the convergence guarantees for the policy gradient for weakly smooth and continuous action space settings. To the best of our knowledge, this is the first work to establish the convergence of policy gradient methods under an unbounded gradient without Lipschitz smoothness conditions. 
Thus, our work significantly generalizes the scope of existing analysis while opening numerous lines of future research. Our assumptions are also practically applicable, as we demonstrate through several examples.

Nonetheless, there are many important limitations for our analysis. Firstly, it is likely that Assumption \ref{as:glob} can be significantly relaxed, as in other recent work \cite{liu2020improved}. A more careful analysis would have more complex dependence on the problem parameters. It may also be interesting to consider weaker assumptions than ergodicity, by adding regularization conditions on the initial distribution of policies. For practical problems, this is often necessary since the smoothness coefficients can be unbounded except in a reasonable starting set. We also believe that weak smoothness can be relaxed further to locally non-smooth problems ($\holder = 0$), by applying smoothing techniques from optimization \cite{nesterov2015universal}. In addition, no practical studies on empirical performance have been done when considering the trade-off between smoothness conditions and convergence rates. Finally, we can quantify the convergence of the distribution of $J(\theta)$ using functionals such as the KL divergence or Wasserstein metric.

\clearpage
\newpage
\printbibliography

@techreport{bellman1954theory,
  title={The theory of dynamic programming},
  author={Bellman, Richard},
  year={1954},
  institution={Rand corp santa monica ca}
}

@article{xu2020improving,
  title={Improving Sample Complexity Bounds for Actor-Critic Algorithms},
  author={Xu, Tengyu and Wang, Zhe and Liang, Yingbin},
  journal={arXiv preprint arXiv:2004.12956},
  year={2020}
}

@article{amari1998natural,
  title={Natural gradient works efficiently in learning},
  author={Amari, Shun-Ichi},
  journal={Neural computation},
  volume={10},
  number={2},
  pages={251--276},
  year={1998},
  publisher={MIT Press}
}

@article{mitrophanov2005sensitivity,
  title={Sensitivity and convergence of uniformly ergodic Markov chains},
  author={Mitrophanov, A Yu},
  journal={Journal of Applied Probability},
  volume={42},
  number={4},
  pages={1003--1014},
  year={2005},
  publisher={Cambridge University Press}
}

@article{rudolf2018perturbation,
  title={Perturbation theory for Markov chains via Wasserstein distance},
  author={Rudolf, Daniel and Schweizer, Nikolaus and others},
  journal={Bernoulli},
  volume={24},
  number={4A},
  pages={2610--2639},
  year={2018},
  publisher={Bernoulli Society for Mathematical Statistics and Probability}
}

@inproceedings{agarwal2020optimality,
  title={Optimality and approximation with policy gradient methods in markov decision processes},
  author={Agarwal, Alekh and Kakade, Sham M and Lee, Jason D and Mahajan, Gaurav},
  booktitle={Conference on Learning Theory},
  pages={64--66},
  year={2020}
}

@article{sidford2018near,
  title={Near-optimal time and sample complexities for solving discounted Markov decision process with a generative model},
  author={Sidford, Aaron and Wang, Mengdi and Wu, Xian and Yang, Lin F and Ye, Yinyu},
  journal={arXiv preprint arXiv:1806.01492},
  year={2018}
}

@article{doya2000reinforcement,
  title={Reinforcement learning in continuous time and space},
  author={Doya, Kenji},
  journal={Neural computation},
  volume={12},
  number={1},
  pages={219--245},
  year={2000},
  publisher={MIT Press}
}

@article{agarwal2020pc,
  title={Pc-pg: Policy cover directed exploration for provable policy gradient learning},
  author={Agarwal, Alekh and Henaff, Mikael and Kakade, Sham and Sun, Wen},
  journal={arXiv preprint arXiv:2007.08459},
  year={2020}
}

@article{cai2019provably,
  title={Provably efficient exploration in policy optimization},
  author={Cai, Qi and Yang, Zhuoran and Jin, Chi and Wang, Zhaoran},
  journal={arXiv preprint arXiv:1912.05830},
  year={2019}
}

@article{yang2019sample,
  title={Sample-optimal parametric q-learning using linearly additive features},
  author={Yang, Lin F and Wang, Mengdi},
  journal={arXiv preprint arXiv:1902.04779},
  year={2019}
}

@article{williams1992simple,
  title={Simple statistical gradient-following algorithms for connectionist reinforcement learning},
  author={Williams, Ronald J},
  journal={Machine learning},
  volume={8},
  number={3-4},
  pages={229--256},
  year={1992},
  publisher={Springer}
}

@phdthesis{kakade2003sample,
  title={On the sample complexity of reinforcement learning},
  author={Kakade, Sham Machandranath and others},
  year={2003}
}

@inproceedings{konda2000actor,
  title={Actor-critic algorithms},
  author={Konda, Vijay R and Tsitsiklis, John N},
  booktitle={Advances in neural information processing systems},
  pages={1008--1014},
  year={2000}
}

@article{sutton1999between,
  title={Between MDPs and semi-MDPs: A framework for temporal abstraction in reinforcement learning},
  author={Sutton, Richard S and Precup, Doina and Singh, Satinder},
  journal={Artificial intelligence},
  volume={112},
  number={1-2},
  pages={181--211},
  year={1999},
  publisher={Elsevier}
}

@article{kakade2001natural,
  title={A natural policy gradient},
  author={Kakade, Sham M},
  journal={Advances in neural information processing systems},
  volume={14},
  pages={1531--1538},
  year={2001}
}

@article{fazel2018global,
  title={Global convergence of policy gradient methods for linearized control problems},
  author={Fazel, Maryam and Ge, Rong and Kakade, Sham M and Mesbahi, Mehran},
  year={2018}
}

@inproceedings{sutton2000policy,
  title={Policy gradient methods for reinforcement learning with function approximation},
  author={Sutton, Richard S and McAllester, David A and Singh, Satinder P and Mansour, Yishay},
  booktitle={Advances in neural information processing systems},
  pages={1057--1063},
  year={2000}
}

@article{du2019good,
  title={Is a Good Representation Sufficient for Sample Efficient Reinforcement Learning?},
  author={Du, Simon S and Kakade, Sham M and Wang, Ruosong and Yang, Lin F},
  journal={arXiv preprint arXiv:1910.03016},
  year={2019}
}

@article{van2019comments,
  title={Comments on the du-kakade-wang-yang lower bounds},
  author={Van Roy, Benjamin and Dong, Shi},
  journal={arXiv preprint arXiv:1911.07910},
  year={2019}
}

@article{thomas2017policy,
  title={Policy gradient methods for reinforcement learning with function approximation and action-dependent baselines},
  author={Thomas, Philip S and Brunskill, Emma},
  year={2017}
}

@article{wang2019neural,
  title={Neural policy gradient methods: Global optimality and rates of convergence},
  author={Wang, Lingxiao and Cai, Qi and Yang, Zhuoran and Wang, Zhaoran},
  journal={arXiv preprint arXiv:1909.01150},
  year={2019}
}

@article{liu2019neural,
  title={Neural proximal/trust region policy optimization attains globally optimal policy},
  author={Liu, Boyi and Cai, Qi and Yang, Zhuoran and Wang, Zhaoran},
  journal={arXiv preprint arXiv:1906.10306},
  year={2019}
}

@article{zou2019finite,
  title={Finite-sample analysis for sarsa with linear function approximation},
  author={Zou, Shaofeng and Xu, Tengyu and Liang, Yingbin},
  journal={arXiv preprint arXiv:1902.02234},
  year={2019}
}

@article{devolder2014first,
  title={First-order methods of smooth convex optimization with inexact oracle},
  author={Devolder, Olivier and Glineur, Fran{\c{c}}ois and Nesterov, Yurii},
  journal={Mathematical Programming},
  volume={146},
  number={1},
  pages={37--75},
  year={2014},
  publisher={Springer}
}

@article{polyak1992acceleration,
  title={Acceleration of stochastic approximation by averaging},
  author={Polyak, Boris T and Juditsky, Anatoli B},
  journal={SIAM journal on control and optimization},
  volume={30},
  number={4},
  pages={838--855},
  year={1992},
  publisher={SIAM}
}

@inproceedings{tu2018least,
  title={Least-squares temporal difference learning for the linear quadratic regulator},
  author={Tu, Stephen and Recht, Benjamin},
  booktitle={International Conference on Machine Learning},
  pages={5005--5014},
  year={2018},
  organization={PMLR}
}

@article{kumar2019sample,
  title={On the sample complexity of actor-critic method for reinforcement learning with function approximation},
  author={Kumar, Harshat and Koppel, Alec and Ribeiro, Alejandro},
  journal={arXiv preprint arXiv:1910.08412},
  year={2019}
}

@article{bhandari2019global,
  title={Global optimality guarantees for policy gradient methods},
  author={Bhandari, Jalaj and Russo, Daniel},
  journal={arXiv preprint arXiv:1906.01786},
  year={2019}
}

@article{deng2016deep,
  title={Deep direct reinforcement learning for financial signal representation and trading},
  author={Deng, Yue and Bao, Feng and Kong, Youyong and Ren, Zhiquan and Dai, Qionghai},
  journal={IEEE transactions on neural networks and learning systems},
  volume={28},
  number={3},
  pages={653--664},
  year={2016},
  publisher={IEEE}
}

@article{yu2019reinforcement,
  title={Reinforcement learning in healthcare: A survey},
  author={Yu, Chao and Liu, Jiming and Nemati, Shamim},
  journal={arXiv preprint arXiv:1908.08796},
  year={2019}
}

@article{kober2013reinforcement,
  title={Reinforcement learning in robotics: A survey},
  author={Kober, Jens and Bagnell, J Andrew and Peters, Jan},
  journal={The International Journal of Robotics Research},
  volume={32},
  number={11},
  pages={1238--1274},
  year={2013},
  publisher={SAGE Publications Sage UK: London, England}
}

@article{nesterov2015universal,
  title={Universal gradient methods for convex optimization problems},
  author={Nesterov, Yu},
  journal={Mathematical Programming},
  volume={152},
  number={1},
  pages={381--404},
  year={2015},
  publisher={Springer}
}

@article{yashtini2016global,
  title={On the global convergence rate of the gradient descent method for functions with H{\"o}lder continuous gradients},
  author={Yashtini, Maryam},
  journal={Optimization letters},
  volume={10},
  number={6},
  pages={1361--1370},
  year={2016},
  publisher={Springer}
}

@article{chen2016statistical,
  title={Statistical inference for model parameters in stochastic gradient descent},
  author={Chen, Xi and Lee, Jason D and Tong, Xin T and Zhang, Yichen},
  journal={arXiv preprint arXiv:1610.08637},
  year={2016}
}

@inproceedings{lakshminarayanan2018linear,
  title={Linear stochastic approximation: How far does constant step-size and iterate averaging go?},
  author={Lakshminarayanan, Chandrashekar and Szepesvari, Csaba},
  booktitle={International Conference on Artificial Intelligence and Statistics},
  pages={1347--1355},
  year={2018},
  organization={PMLR}
}

@article{jain2018parallelizing,
  title={Parallelizing stochastic gradient descent for least squares regression: mini-batching, averaging, and model misspecification},
  author={Jain, Prateek and Kakade, Sham and Kidambi, Rahul and Netrapalli, Praneeth and Sidford, Aaron},
  journal={Journal of Machine Learning Research},
  volume={18},
  year={2018}
}

@book{kushner2003stochastic,
  title={Stochastic approximation and recursive algorithms and applications},
  author={Kushner, Harold and Yin, G George},
  volume={35},
  year={2003},
  publisher={Springer Science \& Business Media}
}

@article{ferre2013regular,
  title={Regular perturbation of V-geometrically ergodic Markov chains},
  author={Ferr{\'e}, D{\'e}borah and Herv{\'e}, Loïc and Ledoux, James},
  journal={Journal of applied probability},
  volume={50},
  number={1},
  pages={184--194},
  year={2013},
  publisher={Cambridge University Press}
}

@article{hoeg2020regularity,
  title={Regularity of solutions of the parabolic normalized p-Laplace equation},
  author={H{\o}eg, Fredrik Arbo and Lindqvist, Peter},
  journal={Advances in Nonlinear Analysis},
  volume={9},
  number={1},
  pages={7--15},
  year={2020},
  publisher={De Gruyter}
}

@article{sciunzi2014regularity,
  title={Regularity and comparison principles for p-Laplace equations with vanishing source term},
  author={Sciunzi, Berardino},
  journal={Communications in Contemporary Mathematics},
  volume={16},
  number={06},
  pages={1450013},
  year={2014},
  publisher={World Scientific}
}

@book{lindqvist2017notes,
  title={Notes on the p-Laplace equation},
  author={Lindqvist, Peter},
  number={161},
  year={2017},
  publisher={University of Jyv{\"a}skyl{\"a}}
}

@inproceedings{chou2017improving,
  title={Improving stochastic policy gradients in continuous control with deep reinforcement learning using the beta distribution},
  author={Chou, Po-Wei and Maturana, Daniel and Scherer, Sebastian},
  booktitle={International conference on machine learning},
  pages={834--843},
  year={2017},
  organization={PMLR}
}

@article{papini2019smoothing,
  title={Smoothing policies and safe policy gradients},
  author={Papini, Matteo and Pirotta, Matteo and Restelli, Marcello},
  journal={arXiv preprint arXiv:1905.03231},
  year={2019}
}

@inproceedings{liu2020improved,
  title={An Improved Analysis of (Variance-Reduced) Policy Gradient and Natural Policy Gradient Methods.},
  author={Liu, Yanli and Zhang, Kaiqing and Basar, Tamer and Yin, Wotao},
  booktitle={NeurIPS},
  year={2020}
}

@article{zhang2020global,
  title={Global convergence of policy gradient methods to (almost) locally optimal policies},
  author={Zhang, Kaiqing and Koppel, Alec and Zhu, Hao and Basar, Tamer},
  journal={SIAM Journal on Control and Optimization},
  volume={58},
  number={6},
  pages={3586--3612},
  year={2020},
  publisher={SIAM}
}

@article{mei2021leveraging,
  title={Leveraging non-uniformity in first-order non-convex optimization},
  author={Mei, Jincheng and Gao, Yue and Dai, Bo and Szepesvari, Csaba and Schuurmans, Dale},
  journal={arXiv preprint arXiv:2105.06072},
  year={2021}
}

@article{zhang2020variational,
  title={Variational policy gradient method for reinforcement learning with general utilities},
  author={Zhang, Junyu and Koppel, Alec and Bedi, Amrit Singh and Szepesvari, Csaba and Wang, Mengdi},
  journal={arXiv preprint arXiv:2007.02151},
  year={2020}
}

@inproceedings{mei2020global,
  title={On the global convergence rates of softmax policy gradient methods},
  author={Mei, Jincheng and Xiao, Chenjun and Szepesvari, Csaba and Schuurmans, Dale},
  booktitle={International Conference on Machine Learning},
  pages={6820--6829},
  year={2020},
  organization={PMLR}
}

@article{zhang2021convergence,
  title={On the convergence and sample efficiency of variance-reduced policy gradient method},
  author={Zhang, Junyu and Ni, Chengzhuo and Yu, Zheng and Szepesvari, Csaba and Wang, Mengdi},
  journal={arXiv preprint arXiv:2102.08607},
  year={2021}
}

@article{brockman2016openai,
  title={Openai gym},
  author={Brockman, Greg and Cheung, Vicki and Pettersson, Ludwig and Schneider, Jonas and Schulman, John and Tang, Jie and Zaremba, Wojciech},
  journal={arXiv preprint arXiv:1606.01540},
  year={2016}
}

@article{coscia1999holder,
  title={H{\"o}lder continuity of the gradient of p (x)-harmonic mappings},
  author={Coscia, Alessandra and Mingione, Giuseppe},
  journal={Comptes Rendus de l'Acad{\'e}mie des Sciences-Series I-Mathematics},
  volume={328},
  number={4},
  pages={363--368},
  year={1999},
  publisher={Elsevier}
}

@article{mao2020perturbation,
  title={Perturbation theory and uniform ergodicity for discrete-time Markov chains},
  author={Mao, Yonghua and Song, Yanhong},
  journal={arXiv preprint arXiv:2003.06978},
  year={2020}
}

\clearpage
\newpage

\newpage
\appendix
The structure of the appendix is as follows. First we show the weak smoothness of the objective function $J(\theta)$. Then, we prove Theorem \ref{th:con} and Theorem \ref{th:opt} respectively, as well as their respective corollaries. In the final appendix, we state the details for our numerical demonstrations.

\section{Weak Smoothness of Objective}

    Recall the definitions of $\Prob^n_{\theta}$ as the $n$-step state transition kernel of the Markov chain induced by the policy $\pi_\theta$. We denote $\Prob_{\theta} \triangleq \Prob_{\theta}^1$ as a shorthand.
    In the following, we refer to a Markov Chain as a pair $(\mathcal{S}, \mathbb{Q})$, where the first item is the state-space and the second item is the (time-invariant) transition function. 
    
\subsection{Visitation Distribution Smoothness}
    Now we bound the differences between the visitation distributions of two policies. We first introduce a lemma from the theory of Markov Chains, connecting ergodicity with the perturbation bounds on the $n$-step transition kernels.
    \begin{lemma} \label{lem:ergodicity_to_visitation}
    (Adapted from \cite{mao2020perturbation}, Theorem 2.1) Let the Markov Chains generated by $(\mathcal{S}, \mathbb{P}_\theta)$, $\mathcal{S}$ satisfy the uniform ergodicity assumption, Assumption \ref{as:erg}. Then for any initial distribution $\rho$,
        \begin{align*}
        \int \Exp_{s_0 \sim \rho}\sq{\abs{\Prob_{\theta}^n(s| s_0) - \Prob_{\theta+\eta}^n( s|s_0)}} \, ds \leq \frac{C_0}{1-\delta} D(\Prob_{\theta}, \Prob_{\theta+\eta}),
        \end{align*}
        where $D(\mathbb{Q}, \mathbb{R}) = \sup_s \int \abs{\mathbb{Q}(s'|s) - \mathbb{R}(s'|s)} \, ds'$ is a measure of distance between two kernels, and where $C_0, \delta$ are the constants from Assumption \ref{as:erg}.
    \end{lemma}
    
    \begin{lemma} \label{lemma:visit_contration}
    Let $\rho$ be any probability distribution on $\State$, and $\pi$. Then, for parameters $\theta, \eta \in \mathbb{R}^n$
    \begin{align*}
        \norm{d_\theta^\rho - d_{\theta+\eta}^\rho} \leq L_d \norm{\eta}^{\frac{\beta_1}{2}},
        \end{align*}   
    where $L_d = \frac{(1+C_0-\delta) \sqrt{ 2C_{\nu, 1}}}{(1-\delta)}$, with $C_0, \delta$ from Assumption \ref{as:erg} and $C_{\nu,1}, \beta_1$ from Assumption \ref{as:smooth}.
    \end{lemma}
    \textbf{Proof:} 
    Let us factor the state distribution as the following:
    \begin{equation*}
        d_\theta^\rho(s',a) = H_{\theta}^{\rho}(s') \pi_\theta(a|s').
    \end{equation*}
    
    Therefore we can bound the difference of two visitation distributions as:
    \begin{equation*}
    \begin{split}
        \int \int \abs{d_\theta^\rho(s,a) - d_{\theta+\eta}^\rho(s,a)} \, da \, ds &\leq \int \int \abs{H_{\theta}^{\rho}(s) \pi_\theta(a|s) - H_{\theta+\eta}^{\rho}(s) \pi_{\theta+\eta}(a|s)} \, da \, ds  \\ 
        & \qquad + \int \int \abs{ H_{\theta+\eta}^{\rho}(s) \pi_\theta(a|s) - H_{\theta+\eta}^{\rho}(s) \pi_{\theta+\eta}(a|s)} \, da\, ds \\
        &\leq \int  \abs{H_{\theta}^{\rho}(s) - H_{\theta+\eta}^{\rho}(s)} \int \pi_{\theta}(a|s) \, da \, ds \\
        &\qquad + \int H_{\theta}^{\rho}(s) \int \abs{\pi_{\theta}(a|s) - \pi_{\theta+\eta}(a|s)} \, da\, ds \\
        &\leq \norm{H_{\theta}^{\rho} - H_{\theta+\eta}^{\rho}}_{TV} + \sup_{s} \int \abs{\pi_{\theta}(a|s) - \pi_{\theta+\eta}(a|s)} \, da.
    \end{split} 
    \end{equation*}
    The second term can be bounded using Pinsker's inequality and Assumption \ref{as:smooth}, obtaining
    \begin{align*}
        \sup_{s} \int \abs{\pi_{\theta}(a|s) - \pi_{\theta+\eta}(a|s)} \, da &\leq \sqrt{2} \sup_s \sqrt{\int \pi_{\theta}(a|s) \log \frac{\pi_{\theta}(a|s)}{\pi_{\theta+\eta}(a|s)} \,  da} \\
        &\leq \sqrt{2 C_{\nu, 1}} \norm{\eta}^{\beta_1/2}.
    \end{align*}
    For the other term, by Lemma \ref{lem:ergodicity_to_visitation}, we know that for all $n > 0$ under Assumption \ref{as:erg}
    \begin{equation*}
        \int \abs{\Exp_{s_0 \sim \rho}\sq{\Prob_{\theta}^n(s| s_0) - \Prob_{\theta+\eta}^n( s|s_0)}} \, ds \leq \frac{C_0}{1-\delta} D(\Prob_{\theta}, \Prob_{\theta+\eta}).
    \end{equation*}
    
    Under these assumptions, the norm of the visitation distribution under perturbation can be bounded by the norm of the perturbation kernel using Lemma \ref{lemma:visit_contration}:
    \begin{align*}
        \norm{H_{\theta}^\rho - H_{\theta+\eta}^\rho}_{TV} &\leq \int (1-\gamma) \sum_{n=1}^{\infty} \gamma^n \abs{\Exp_{s_0 \sim \rho}[\Prob^n_{\theta}(s|s_0) - \Prob^n_{\theta+\eta}(s|s_0)]} \, ds \\
        &= (1-\gamma) \sum_{n=1}^{\infty} \gamma^n \int \abs{\Exp_{s_0 \sim \rho} \sq{\Prob_{\theta}^n(s|s_0) - \Prob_{\theta+\eta, \rho}^n(s|s_0)}} \, ds \\
        &\leq \frac{C_0}{1-\delta}  D(\Prob_{\theta}, \Prob_{\theta+\eta})  \\
        &= \sup_s \frac{C_0}{1-\delta} \bracket{\int \int P(s'|s,a) \abs{\pi_{\theta}(a|s) - \pi_{\theta+\eta}(a|s)} \, ds' \, da}  \\
        &\overset{(i)}{=} \sup_s \frac{C_0}{1-\delta} \int \abs{\pi_{\theta}(a|s) - \pi_{\theta+\eta}(a|s)} \,  da \\
        &\overset{(ii)}{\leq} \frac{C_0 \sqrt{2 C_{\nu, 1}}}{(1-\delta)} \norm{\eta}^{\frac{\beta_1}{2}},
    \end{align*}
    where $(i)$ follows as $\int P(s'|s,a) \, ds' = 1$ for any $s,a$, $(ii)$ is again using Pinsker's inequality and Assumption \ref{as:smooth}.
    Returning to our overall bound, we get:
    \begin{equation*}
    \begin{split}
        \norm{d_\theta^\rho - d_{\theta+\eta}^\rho}_{TV} &\leq \norm{H_{\theta}^{\rho} - H_{\theta+\eta}^{\rho}}_{TV} + \sup_{s} \int \abs{\pi_{\theta}(a|s) - \pi_{\theta+\eta}(a|s)} \, da
        \\ &\leq  \frac{C_0 \sqrt{2 C_{\nu, 1}}}{(1-\delta)} \norm{\eta}^{\frac{\beta_1}{2}} +  \sqrt{2 C_{\nu, 1}} \norm{\eta}^{\frac{\beta_1}{2}} \\
        &\leq  \frac{(1+C_0-\delta) \sqrt{2 C_{\nu, 1}}}{(1-\delta)} \norm{\eta}^{\frac{\beta_1}{2}}.
    \end{split}
    \end{equation*}
    This concludes the proof. \hfill $\square$

    \subsection{Q-function Analysis}
    
    \begin{lemma}\label{lem:q_smooth}
        Let $\rho$ be any probability distribution on $\mathcal{S}$, and let the policy class and the MDP satisfy Assumptions \ref{as:smooth}, \ref{as:erg}. Then for parameters $\theta, \eta \in \mathbb{R}^n$
        \begin{equation*}
        \int \int d_\theta^\rho(s,a) \abs{Q_{\theta}(s,a) - Q_{\theta + \eta}(s,a)} \, da \, ds \leq \frac{L_Q}{1-\gamma} \norm{\eta}^{\frac{\beta_1}{2}},
        \end{equation*}
        where $L_Q = \frac{\alpha \gamma (1+C_0-\delta) \sqrt{2 C_{\nu, 1}}}{(1-\delta)}$, with $C_0, \delta$ from Assumption \ref{as:erg} and $C_{\nu,1}$ from Assumption \ref{as:smooth}, $\alpha$ bounds the absolute value of the reward, and $\gamma$ is the discount factor.
    \end{lemma}
    Define the state distribution with the density $\tilde \rho_{s,a}(s') = P(s'|a,s)$ as the state distribution after taking action $a$ in state $s$ 
    We can write the definition of the Q-function as:
    \begin{align*}
        Q_\theta(s,a) &= \int r R(r|s,a) \, dr + \gamma \int \int \int r \tilde \rho_{s,a}(s') \pi_\theta(a'|s') R(r|s',a') \, dr \, da' \, ds' \\
        &\qquad +  \sum_{n=2}^{\infty} \gamma^{n} \int \int \int r \Exp_{s_0 \sim \tilde \rho_{s,a}}\sq{\Prob_{\theta}^{n-1}(s'|s_0)} \pi_{\theta}(a'|s') R(r|s', a') \, dr \, da' \, ds' \\
        &= \int r R(r|s,a) \, dr + \frac{\gamma}{1-\gamma} \int \int \int r  R(r|s', a') d_{\theta}^{\tilde \rho_{s,a}}(s',a') \, dr \, da' \, ds',
    \end{align*}
    Subsequently the difference reduces to:
    \begin{align*}
    \abs{Q_{\theta}(s,a) - Q_{\theta + \eta}(s,a)} &\leq  \frac{\gamma}{1-\gamma} \int \int \int \abs{r}  R(r|s', a') \abs{d_{\theta}^{\tilde \rho_{s,a}}(s',a') - d_{\theta+\eta}^{\tilde \rho_{s,a}}(s',a')} \, dr \, da' \, ds' \\
    &\overset{(i)}{\leq} \frac{\alpha \gamma}{1-\gamma} \int \int \abs{d_{\theta}^{\tilde \rho_{s,a}}(s',a') - d_{\theta+\eta}^{\tilde \rho_{s,a}}(s',a')} \, da' \, ds',
    \end{align*}
    where in $(i)$ we recall that the reward distribution has bounded support on $[-\alpha, \alpha]$.
    This then reduces to Lemma \ref{lemma:visit_contration}, under the initial distribution $\tilde \rho_{s,a}$. From our earlier bound we obtain the following:
    \begin{equation*}
        \begin{split}
        \abs{Q_{\theta}(s,a) - Q_{\theta + \eta}(s,a)} &\leq \frac{\alpha \gamma L_d}{1-\gamma} \norm{\eta}^{\frac{\beta_1}{2}}.
        \end{split}
    \end{equation*}
    Finally, we can plug this bound into the integral to get
    \begin{align*}
        \int \int d_\theta^\rho(s,a) \abs{Q_{\theta}(s,a) - Q_{\theta + \eta}(s,a)} \, da \, ds \leq \frac{\alpha \gamma L_d}{1-\gamma} \norm{\eta}^{\frac{\beta_1}{2}} \int \int d_{\theta}^{\rho}(s,a) \, ds \, da = \frac{\alpha \gamma L_d}{1-\gamma} \norm{\eta}^{\frac{\beta_1}{2}}.
    \end{align*}
    \hfill $\square$
    
     \subsection{Bounding Score Function}
    \begin{lemma}
        Under Assumption \ref{as:smooth}, the score function is bounded as
        \begin{align*}
            \int \int d_\theta^\rho(s,a) \norm{\psi_\theta(s,a) - \psi_{\theta+\eta}(s,a)} \, ds \, da \leq C_{\nu, 2} \norm{\eta}^{\beta_2},
        \end{align*}
        where the constants $C_{\nu,2}, \beta_2$ are given in Assumption \ref{as:smooth}.
    \end{lemma}
    
    \textbf{Proof:} This follows from the second bound in Assumption \ref{as:smooth}.
    \begin{equation*}
    \begin{split}
        &\int \int d_\theta^\rho(s,a) \norm{\psi_\theta(s,a) - \psi_{\theta+\eta}(s,a)} \, da\, ds \leq C_{\nu, 2} \norm{\eta}^{\beta_2} \int \int d_\theta^\rho(s,a) \, da \, ds = C_{\nu, 2} \norm{\eta}^{\beta_2}.
    \end{split}
    \end{equation*}
    \hfill $\square$

\subsection{Weak Smoothness of Objective}
We conclude by proving the weak smoothness property of the loss function.

\begin{prop} \label{prop:smooth}
For any $\theta, \eta \in \R^d$, $\nabla J(\theta)$ satisfies the following smoothness property
\begin{align*}
\norm{\nabla J(\theta) - \nabla J(\theta+\eta)} \leq \frac{L_1}{1-\gamma} \norm{\eta}^{\beta_1/4} + \frac{L_2}{1-\gamma} \norm{\eta}^{\beta_2},
\end{align*}
where $L_1 =  \frac{2\alpha C_{\nu,1}^{1/4}\sqrt{ \psi_\infty (1+C_0-\delta) } \left(\sqrt{\gamma} + \sqrt{D_{\infty}}\right)}{1-\delta}$ and $L_2 = \alpha C_{\nu,2}$. Here $\gamma$ is the discount factor, $\beta_1, \beta_2$ are the smoothness coefficients defined in Assumption \ref{as:smooth}, $D_{\infty} \triangleq \sup_{\theta_1, \theta_2 \in \Theta} \norm{\frac{d_{\theta_1}^\rho}{d_{\theta_2}^{\rho}}}_\infty + 1$ is a distribution mismatch coefficient, $C_0, \delta$ are defined in Assumption \ref{as:erg}, $C_{\nu,1}, C_{\nu,2}$ are defined in Assumption \ref{as:smooth}, $\psi_{\infty}$ is defined in Assumption \ref{as:moment}, and $\alpha$ is the bound on the reward.
\end{prop}

\textbf{Proof:} Decomposing $\nabla J(\theta) = \Exp_{(s,a) \sim d_{\theta}^\rho} [Q_{\pi_\theta}(s,a) \psi_\theta(s,a)]$ and using the triangle inequality and Jensen's inequality, we get
\begin{align*}
    \norm{\nabla J(\theta) - \nabla J(\theta + \eta)} &= \Bigl \| \int \int Q_{\theta}(s,a) \psi_{\theta}(s,a) d^\rho_{\theta}(s,a)  \, ds \, da \\
    &\qquad - \int \int Q_{\theta + \eta}(s,a) \psi_{\theta + \eta}(s,a) d^\rho_{\theta+\eta}(s,a) \, ds \, da \Bigr\| \\
    &\leq \int \int  \abs{Q_{\theta}(s,a)-Q_{\theta+\eta}(s,a)} \norm{\psi_{\theta}(s,a)} d_{\theta}^{\rho}(s,a) \, ds \, da \\
        & \quad + \int \int \abs{Q_{\theta+\eta}(s,a)} \norm{\psi_\theta(s,a) - \psi_{\theta+\eta}(s,a)}  d_\theta^\rho(s,a) \, ds \, da 
        \\
        & \quad + \int \int \abs{Q_{\theta+\eta}(s,a)}  \norm{\psi_{\theta+\eta}(s,a)}\abs{d_\theta^\rho(s,a) - d_{\theta+\eta}^\rho(s,a)} \, ds \, da.
\end{align*}
We handle each of these terms separately. For the first term,
\begin{align*}
    &\int \int  \abs{Q_{\theta}(s,a)-Q_{\theta+\eta}(s,a)} \norm{\psi_{\theta}(s,a)} d_{\theta}^{\rho}(s,a) \, ds \, da \\
    &\qquad \leq \sqrt{\int \int \abs{Q_{\theta}(s,a)-Q_{\theta+\eta}(s,a)}^2 d_{\theta}^\rho(s,a) \, ds \, da} \times \sqrt{\int \int \norm{\psi_\theta(s,a)}^2 d_{\theta}^\rho(s,a) \,ds \, da} \\
    &\overset{(i)}{\leq} \sqrt{\frac{2 \psi_\infty \alpha}{1-\gamma} \times \int \int \abs{Q_{\theta}(s,a)-Q_{\theta+\eta}(s,a)} d_{\theta}^\rho(s,a) \, ds \, da} \\
    &\overset{(ii)}{\leq} \sqrt{2 \psi_\infty \frac{\alpha L_Q}{(1-\gamma)^2} \norm{\eta}^{\beta_1/2}} = \frac{\sqrt{2 \psi_{\infty} \alpha L_Q}}{1-\gamma} \norm{\eta}^{\beta_1/4},
\end{align*}
where in $(i)$ we use Assumption \ref{as:moment} and that $\norm{Q_{\theta}}_\infty \leq \frac{\alpha}{1-\gamma}$ by boundedness of the reward for any $\theta$, and in $(ii)$ we use Lemma \ref{lem:q_smooth}. For the second term, we use
\begin{align*}
     &\int \int \abs{Q_{\theta+\eta}(s,a)} \norm{\psi_\theta(s,a) - \psi_{\theta+\eta}(s,a)}  d_\theta^\rho(s,a) \, ds \, da \\
     &\qquad \leq \norm{Q_{\theta+\eta}}_\infty \int \int d_\theta^\rho(s,a) \norm{\psi_\theta(s,a) - \psi_{\theta+\eta}(s,a)} \, ds \, da \\
     &\qquad \leq \frac{\alpha C_{\nu, 2}}{1-\gamma} \norm{\eta}^{\beta_2}.
\end{align*}
For the final term, we get
\begin{align*}
    &\int \int \abs{Q_{\theta+\eta}(s,a)}  \norm{\psi_{\theta+\eta}(s,a)} \abs{d_\theta^\rho(s,a) - d_{\theta+\eta}^\rho(s,a)} \, ds \, da \\
    &\qquad \leq \int \int \abs{Q_{\theta+\eta}(s,a)}  \norm{\psi_{\theta+\eta}(s,a)}\abs{\frac{d_\theta^\rho(s,a)}{d_{\theta+\eta}^\rho(s,a)} - 1} d_{\theta+\eta}^\rho(s,a) \, ds \, da \\
    &\qquad \leq \norm{Q_{\theta+\eta}}_\infty \sqrt{\int \int \norm{\psi_{\theta+\eta}(s,a)}^2 d_{\theta+\eta}^\rho(s,a) \, ds \, da} \times \sqrt{\int \int \abs{\frac{d_\theta^\rho(s,a)}{d_{\theta+\eta}^\rho(s,a)} - 1}^2 d_{\theta+\eta}^\rho(s,a) \, ds \, da} \\
    &\qquad \leq \frac{\sqrt{\psi_{\infty}} \alpha }{1-\gamma} \sqrt{\norm{\frac{d_{\theta}^\rho(s,a)}{d_{\theta+\eta}^{\rho}(s,a)}}_\infty + 1} \times  \sqrt{\int \int \abs{d_\theta^\rho(s,a) - d_{\theta+\eta}^{\rho}(s,a)} \,ds \, da} \\
    &\qquad \leq \frac{\sqrt{D_{\infty} \psi_{\infty} L_d} \alpha }{1-\gamma} \norm{\eta}^{\beta_1/4},
\end{align*}
where we denote $D_{\infty} \triangleq \sup_{\theta_1, \theta_2 \in \Theta} \norm{\frac{d_{\theta_1}^\rho}{d_{\theta_2}^{\rho}}}_\infty + 1$ as the distribution mismatch coefficient.

Putting all of this together, we get a final bound of:
    \begin{equation*}
        \begin{split}
            \norm{\nabla J(\theta) - \nabla J(\theta+\eta)} &\leq \frac{\sqrt{2 \psi_{\infty} \alpha L_Q}}{1-\gamma} \norm{\eta}^{\beta_1/4} + \frac{\alpha C_{\nu, 2}}{1-\gamma} \norm{\eta}^{\beta_2} + \frac{\sqrt{D_{\infty} \psi_{\infty} L_d} \alpha }{1-\gamma} \norm{\eta}^{\beta_1/4} \\ 
            &\leq \frac{L_1}{1-\gamma}  \norm{\eta}^{\beta_1/4} + \frac{L_2}{1-\gamma} \norm{\eta}^{\beta_2}.
        \end{split}
    \end{equation*}
    \hfill $\square$
    
    \textbf{Remarks:} Note that the applications of Cauchy-Schwarz are not necessary if we simply assume that $\norm{\psi_\theta}$ is bounded almost surely. Consequently we can retain the smoothness parameters $\beta_1/2$. This fact is noted in the main text. Furthermore, we will no longer need to bound the distribution mismatch coefficient, and instead we can directly apply Lemma \ref{lemma:visit_contration}. Thus we obtain in the case $\beta_1 = 2, \beta_2 = 1$ a bound of the form $\norm{\nabla J(\theta) - \nabla J(\theta+\eta)} \leq \frac{L}{1-\gamma} \norm{\eta}$ for some $L$ not dependent on $\gamma, \theta, \eta$, which is a standard result (see e.g. \cite[Proposition 1]{xu2020improving}).
    

\section{Theorem \ref{th:con}}
    \noindent\textit{Proof of Theorem \ref{th:con}:}
    We divide the analysis into the one for vanilla policy gradient, and the one for natural policy gradient respectively. Note that unless otherwise specified the expectations will be taken with respect to all randomness in the optimization procedure.
    
    First, we show a lemma that $v_{t,i}$ in Algorithms 1-2 is an unbiased estimator for $Q$, which we adapt the following Lemma from \cite[Theorem 3.4]{zhang2020global}.
\begin{lemma}
    The sampler for $Q$ is unbiased, i.e. for all $t,i$,
    \begin{align*}
        \Exp \sq{v_{t,i}\mid s_{t,i} = s, a_{t,i} = a} = Q_{\theta_{t-1}}(s, a).
    \end{align*}
\end{lemma}

\begin{proof}
    Let $h$ be distributed on $\mathbb{N}$ as $\text{Geom}(1-\gamma^{1/2})$, and $j \sim \text{Geom}(1-\gamma)$. Then we can write the following 
    \begin{align*}
        &\Exp\sq{ \sum_{u=m}^{m+h} \gamma^{(u-m)/2} r_{u} \mid j=m, s_m =s , a_m = a} \\&=\Exp\sq{ \sum_{k=0}^\infty \text{Pr}(h = k) \sum_{u=m}^{m+k} \gamma^{(u-m)/2} r_{u} \mid j=m, s_m =s , a_m = a} \\
        &= \Exp\sq{ \sum_{u=m}^\infty \text{Pr}(h \geq u-m) \gamma^{(u-m)/2} r_{u} \mid j=m, s_m =s , a_m = a} \\
        &=\sum_{u=m}^\infty \Exp\sq{ \text{Pr}(h \geq u-m) \gamma^{(u-m)/2} r_{u} \mid j=m, s_m =s , a_m = a} \\
        &\overset{(i)}{=}\sum_{u=m}^\infty \Exp\sq{\gamma^{u-m} r_{u} \mid j=m, s_m =s , a_m = a} \\
        &= Q_{\theta_{t-1}}(s,a),
    \end{align*}
where in $(i)$ we use that $h \sim \text{Geom}(1-\gamma^{1/2}).$ Finally, taking the conditional expectation on $j$ of both sides and using that $Q$ does not depend on $m$ completes the proof.
\end{proof}

    To control the noise, we introduce the following basic lemma on the variance of (one-dimensional) random variables:
    \begin{lemma}
        \label{lem:var}
        Let $X, Y$ be two random variables. Then if $\mathsf{Var}(XY) = \Exp \sq{(XY - \Exp\sq{XY})^2}$ and $|Y| \leq C$ almost surely,
        \begin{equation*}
            \mathsf{Var}(XY) \leq 2C^2\Exp\sq{X^2}.
        \end{equation*}
    \end{lemma}
    \begin{proof}
    \begin{align*}
        \mathsf{Var}(XY) &\leq \Exp \sq{(XY)^2} + \Exp \sq{\abs{XY}}^2\\
        &\leq  C^2 \Exp \sq{X^2} + C^2 \Exp \sq{\abs{X}}^2 \\
        &\leq 2C^2 \Exp \sq{X^2},
    \end{align*}
    where in the last step we use Jensen's inequality on the second term.
    \end{proof}

Finally, we can show the following lemma concerning the gradient:
    \begin{lemma}
        (Variance of Minibatch Policy Gradient) Consider the noise term 
        \begin{align}
            e_t = \frac{1}{B} \sum_{i=1}^B v_{t,i} \psi_{\theta_{t-1}}(s_{t,i}, a_{t,i}) - g_t,
        \end{align}
        where $s_{t,i}, a_{t,i}$ are i.i.d. sampled from $d_{\theta_{t-1}}^\rho$. We show that its variance is finite:
        \begin{equation*}
            \Exp \sq{\norm{e_t}^2} \leq \frac{\sigma^2}{(1-\gamma)^2 B},
        \end{equation*}
        where $\sigma = 3\alpha \sqrt{\psi_{\infty}}$ is a positive constant controlled by the maximum reward $\alpha$ and the integrability constant $\psi_\infty$ from Assumption \ref{as:moment}.
        \label{lem:noise}
    \end{lemma}

    \begin{proof}
    First let us analyze a single sampled gradient, and take the conditional expectation with $\theta_{t-1}$. Applying Lemma \ref{lem:var} to each component and noting that $s_{t,i}, a_{t,i} \sim d^\rho_{\theta_{t-1}}$,
    \begin{equation*}
        \begin{split}
            \Exp  \sq{\norm{g_t - v_{t,i} \psi_{\theta_{t-1}}(s_{t,i}, a_{t,i})}^2 |\ \theta_{t-1}} &\leq \frac{8\alpha^2}{(1-\gamma)^2} \Exp_{s,a \sim d_{\theta_{t-1}}^\rho} \sq{\norm{\psi_{\theta_{t-1}}(s, a)}^2 \mid \theta_{t-1}} \\
            &\overset{(i)}{\leq} \frac{8\alpha^2 \psi_{\infty}}{(1-\gamma)^2},
        \end{split}
    \end{equation*}
    where in $(i)$ we use Assumption 2. Then, since the samples $v_{t,i} \psi_{\theta_{t-1}}(s_{t,i}, a_{t,i})$ are i.i.d., when analyzing the variance of the sum we are left with
    \begin{align*}
        \Exp  \sq{\norm{g_t - \frac{1}{B} \sum_{i=1}^B v_{t,i} \psi_{\theta_{t-1}}(s_{t,i}, a_{t,i})}^2\bigg|\ \theta_{t-1}} &\leq \frac{8\alpha^2 \psi_\infty}{(1-\gamma)^2 B}.
    \end{align*}
    Taking another expectation with respect to $\theta_{t-1}$ concludes the proof. 
    
    \end{proof}
    
    \textbf{Remarks:} The analysis is much simpler if we assume that $\norm{\psi_\theta(s,a)}_\infty < \infty$; in that case, the quantity $\norm{v_{t,i} \psi_{\theta_t}(s_{t,i}, a_{t,i})}$ is bounded almost surely.
    
    Thus the error terms have variance $\Exp \sq{\norm{e_t}^2} \leq \frac{\sigma^2}{(1-\gamma)^2 B}$, and $\Exp \sq{\norm{e_t}^{c+1}} \leq \left(\frac{\sigma}{(1-\gamma)\sqrt{B}}\right)^{c+1}$ by Jensen's inequality for any $c \in [0,1]$.

    \subsection{Policy Gradient}
    We first analyze the vanilla policy gradient. For convenience and brevity, let $g_t = \nabla_\theta J(\theta_{t-1})$ be shorthand for the exact gradient at time $t$, and $e_t = \frac{1}{B} \sum_{i=1}^B v_{t,i} \psi_{\theta_{t-1}}(s_{t,i}, a_{t,i}) - g_t$ as the error from the sampled gradient, where $v_{t,i}, s_{t,i}, a_{t,i}$ are the sampled estimates from Algorithm 1. From \cite[Algorithm 1]{agarwal2020optimality}, we know that $s_{t,i}, a_{t,i}$ are sampled from $d_{\theta}^\rho$ and $v_{t,i}$ is an unbiased estimate of $Q_\theta(s_{t,i}, a_{t,i})$. Consequently, $\Exp \sq{e_t} = 0$, where the expectation is over all randomness in the sampling process.
    
    Following Proposition \ref{prop:smooth}, we know that the target function $J(\theta)$ satisfies a smoothness condition, i.e. $\norm{\nabla J(\theta) - \nabla J(\theta+\eta)} \leq \frac{1}{1-\gamma}\left(L_1 \norm{\eta}^{\beta_1/4} + L_2 \norm{\eta}^{\beta_2} \right)$.
    We know from our algorithm that $\theta_t - \theta_{t-1} = h_t (g_t + e_t)$. Thus we apply mean value theorem to get:
    \begin{equation*}
    \begin{split}
        & \abs{J(\theta_t) - J(\theta_{t-1}) - \inner{g_t}{h_t (g_t+e_t)}} \\
        &\qquad = \abs{\int_0^1 \inner{\nabla J(\theta_{t-1} + \tau h_t (g_t+e_t)) - g_t}{h_t (g_t+e_t)} d\tau} \\
        &\qquad \leq \norm{h_t (g_t + e_t)} \int_0^1 \abs{\nabla J(\theta_{t-1} + \tau h_t (g_t+e_t)) - g_t} d\tau \\
        &\qquad \leq \frac{1}{1-\gamma} \norm{h_t (g_t+e_t)} \left(L_1\norm{h_t (g_t + e_t)}^{\frac{\beta_1}{4} } \int \tau^{\frac{\beta_1}{4} } d\tau + L_2 \norm{h_t (g_t + e_t)}^{\beta_2} \int \tau^{\beta_2}d\tau \right)  \\
        &\qquad \leq \frac{L_1}{1-\gamma} \norm{h_t (g_t+e_t)}^{\frac{\beta_1}{4}  + 1} + \frac{ L_2}{1-\gamma} \norm{h_t (g_t+e_t)}^{\beta_2 + 1}.
    \end{split}
    \end{equation*}
    Consequently the following holds,
    \begin{equation*}
    \begin{split}
        J(\theta_t) &\geq J(\theta_{t-1}) + \inner{h_t (g_t + e_t)}{g_t} - \frac{ L_1}{1-\gamma} \norm{h_t (g_t+e_t)}^{\frac{\beta_1}{4}  + 1} - \frac{ L_2}{1-\gamma} \norm{h_t (g_t+e_t)}^{\beta_2 + 1} \\
        &\geq J(\theta_{t-1}) + h_t \norm{g_t}^2 + h_t \inner{e_t}{g_t} - \frac{ L_1}{1-\gamma} \norm{h_t (g_t+e_t)}^{\frac{\beta_1}{4}  + 1} - \frac{ L_2}{1-\gamma} \norm{h_t (g_t+e_t)}^{\beta_2 + 1} \\
        &\geq J(\theta_{t-1}) + h_t \norm{g_t}^2 + h_t \inner{e_t}{g_t} - \frac{2L_1}{1-\gamma} h_t^{\frac{\beta_1}{4} +1}\left(\norm{g_t}^{\frac{\beta_1}{4}  + 1} + \norm{e_t}^{\frac{\beta_1}{4} + 1} \right) \\
        &\qquad - \frac{2L_2}{1-\gamma} h_t^{\beta_2 + 1}\left(\norm{ g_t}^{\beta_2 + 1} + \norm{e_t}^{\beta_2 + 1} \right).
    \end{split}
    \end{equation*}
    Rearranging and taking expectation, we get
    \begin{align*}
        h_t \Exp \sq{\norm{g_t}^2} &\leq \Exp\sq{J(\theta_t) - J(\theta_{t-1})} + \frac{ 2L_1}{1-\gamma} h_t^{\frac{\beta_1}{4} +1}\left(\Exp \sq{\norm{g_t}^{\frac{\beta_1}{4}+1}} + \Exp\sq{\norm{e_t}^{\frac{\beta_1}{4} + 1}} \right) \\
        &\qquad + \frac{2L_2}{1-\gamma} h_t^{\beta_2 + 1}\left(\Exp\sq{\norm{ g_t}^{\beta_2 + 1}} + \Exp\sq{\norm{e_t}^{\beta_2 + 1}} \right).
    \end{align*}

    Then, we can substitute our bound and apply Jensen's inequality to obtain
    \begin{align*}
        h_t \Exp \sq{\norm{g_t}^2} &\leq \Exp\sq{J(\theta_t) - J(\theta_{t-1})} + \frac{ 2L_1}{1-\gamma} h_t^{\frac{\beta_1}{4} +1}\left(\Exp \sq{\norm{g_t}^{\frac{\beta_1}{4}+1}} + \left(\frac{\sigma}{(1-\gamma)\sqrt{B}} \right)^{\frac{\beta_1}{4}+1} \right) \\
        &\qquad + \frac{2L_2}{1-\gamma} h_t^{\beta_2 + 1}\left(\Exp\sq{\norm{ g_t}^{\beta_2 + 1}} + \left(\frac{\sigma}{(1-\gamma)\sqrt{B}} \right)^{\beta_2 + 1} \right) \\
        &\leq \Exp\sq{J(\theta_t) - J(\theta_{t-1})} + \frac{ 2(L_1+L_2)}{1-\gamma} \Biggl( h_t^{\frac{\beta_1}{4} +1}\left(\Exp \sq{\norm{g_t}^{\frac{\beta_1}{4}+1}} + \left(\frac{\sigma}{(1-\gamma)\sqrt{B}} \right)^{\frac{\beta_1}{4}+1} \right) \\
        &\qquad + h_t^{\beta_2 + 1}\left(\Exp\sq{\norm{ g_t}^{\beta_2 + 1}} + \left(\frac{\sigma}{(1-\gamma)\sqrt{B}} \right)^{\beta_2 + 1} \right)\Biggr). \\
    \end{align*}
    Summing this from $t=1 \ldots T$, and using that $J_*$ is the supremum of $J$, concludes the proof with $C_{PG,1} = 1$ and $C_{PG,2} = 2(L_1 + L_2)$ for policy gradient. \hfill$\square$
    \\

    \noindent
    \textit{Proof of Corollary \ref{cor:convergence}, PG}:
    
    \textbf{Learning Rate II:} Now consider the learning rate, $h_t = \lambda T^{\frac{\holde-1}{\holde+1}}$, when $\holde \neq 1$. (For the case $\holde = 1$, see the constant step-size analysis below). In that case, we get the following bound of convergence, using that $\lambda \leq 1$ and $T^{\holde-1} \leq 1$
    \begin{align*}
        & h_t \norm{g_t}^2 - \frac{2(L_1+L_2)}{1-\gamma} \left(h_t^{\frac{\beta_1}{4} +1}\norm{g_t}^{\frac{\beta_1}{4}+1} + h_t^{\beta_2 + 1}\norm{ g_t}^{\beta_2 + 1}\right)   \\
        &\geq \lambda T^{\frac{\holde-1}{\holde+1}} \norm{g_t}^2 - \frac{2(L_1+L_2)\lambda^{\holde+1}}{1-\gamma} T^{\holde-1} \left( \norm{g_t}^{\frac{\beta_1}{4} + 1} + \norm{ g_t}^{\beta_2 + 1}\right) \\
        &= \lambda T^{\frac{\holde-1}{\holde+1}} \norm{g_t}^{\holde+1} \left( \norm{g_t}^{1-\holde} - \frac{2(L_1 + L_2)\lambda^\holde}{1-\gamma} T^{\frac{\holde(\holde -1)}{\holde+1}} (1+ \norm{g_t}^{\maxholde - \holde}) \right) \\
        &\geq \lambda T^{\frac{\holde-1}{\holde+1}} \norm{g_t}^{\holde+1} \left( \frac{1}{2} \norm{g_t}^{1-\holde} - \frac{4(L_1 + L_2)\lambda^\holde}{1-\gamma} T^{\frac{\holde(\holde -1)}{\holde+1}} \right).
    \end{align*}
    The last step follows since if $\norm{g_t} \geq 1$, $\frac{2(L_1 + L_2)\lambda^\holde}{1-\gamma} T^{\frac{\holde(\holde -1)}{\holde+1}} \norm{g_t}^{\maxholde - \holde} \leq \frac{1}{2} \norm{g_t}^{1 - \holde}$ as long as $\lambda^{-\holde} \geq \frac{4(L_1 + L_2)}{1-\gamma}$, and if $\norm{g_t} < 1$ then $\norm{g_t}^{\maxholde - \holde} < 1$ as well.
    
    Now consider the following cases: either $\norm{g_t}^{1-\holde} \geq \frac{16(L_1 + L_2)\lambda^\holde}{1+\gamma} T^{\frac{\holde(\holde -1)}{\holde+1}}$, in which case the following holds:
    \begin{align*}
        \lambda T^{\frac{\holde-1}{\holde+1}} \norm{g_t}^{\holde+1} \left( \frac{1}{2} \norm{g_t}^{1-\holde} - \frac{4(L_1 + L_2)\lambda^\holde}{1-\gamma} T^{\frac{\holde(\holde -1)}{\holde+1}} \right)\geq \frac{\lambda}{4} T^{\frac{\holde-1}{\holde+1}} \norm{g_t}^2,
    \end{align*}
    or $\norm{g_t}^{1-\holde} <  \frac{16(L_1 + L_2)\lambda^\holde}{\holde+1} T^{\frac{\holde(\holde -1)}{\holde+1}}$, in which case we have the following:
    \begin{align*}
        &\lambda T^{\frac{\holde-1}{\holde+1}} \norm{g_t}^{\holde+1} \left( \frac{1}{2} \norm{g_t}^{1-\holde} - \frac{4(L_1+L_2)\lambda^\holde}{1-\gamma} T^{\frac{\holde(\holde -1)}{\holde+1}} \right) \\
        &\geq \frac{\lambda}{2} T^{\frac{\holde-1}{\holde+1}} \norm{g_t}^2 - \left(\frac{16(L_1 + L_2)\lambda^\holde}{1-\gamma} T^{\frac{\holde(\holde -1)}{\holde+1}}\right)^{\frac{\holde+1}{1-\holde}} \left(\frac{4(L_1 + L_2)\lambda^{\holde+1}}{1-\gamma} T^{\holde-1} \right) \\
        &\geq \frac{\lambda}{2} T^{\frac{\holde-1}{\holde+1}} \norm{g_t}^2 -  \frac{\lambda}{4} T^{-1},
    \end{align*}
    if we choose $\lambda^{-\holde} \geq \frac{16(L_1 + L_2)}{1-\gamma}$. In either case, we can use the worst case bound
    \begin{align*}
        \frac{\lambda}{4} T^{\frac{\holde-1}{\holde+1}} \norm{g_t}^2 -  \frac{\lambda}{4} T^{-1}.
    \end{align*}
    Thus, if $B \geq \sigma^2 (1-\gamma)^{-2}$, this implies the following inequality:
    
    \begin{align*}
        \sum_{t=1}^T \frac{\lambda}{4} T^{\frac{\holde-1}{\holde+1}} \Exp \sq{\norm{g_t}^2} &\leq J_* - J(\theta_0) + \frac{\lambda}{4} +\frac{4(L_1 + L_2)\lambda^{\holde+1}}{1-\gamma} T^{\holde} \left(\frac{\sigma}{(1-\gamma) \sqrt{B}}\right)^{\holde+1}.
    \end{align*}
    After rearranging, we get
    \begin{align*}
        \frac{1}{T} \sum_{t=1}^T \Exp \sq{\norm{g_t}^2} &\leq \frac{4}{\lambda} T^{-\frac{2\holde}{\holde+1}} \left(J_* - J(\theta_0) + \frac{\lambda}{4} \right) \\
        &\qquad + T^{\frac{\holde(\holde-1)}{\holde+1}} \left(\frac{\sigma}{(1-\gamma) \sqrt{B}}\right)^{\holde+1}.
    \end{align*}
    
    \textbf{Learning Rate I:}
    Now consider the time-independent learning rate, $h_t = \lambda$ for some $\lambda > 0$. We follow the same steps as for the previous learning rate, only briefly sketching the argument below. Under the same assumption that $\lambda^{-\holde} \geq \frac{4(L_1 + L_2)}{1-\gamma}$
    \begin{align*}
        & h_t \norm{g_t}^2 - \frac{2(L_1+L_2)}{1-\gamma} \left(h_t^{\frac{\beta_1}{4} +1}\norm{g_t}^{\frac{\beta_1}{4}+1} + h_t^{\beta_2 + 1}\norm{ g_t}^{\beta_2 + 1}\right)  \\
        &\geq \lambda \norm{g_t}^{\holde+1} \left( \frac{1}{2} \norm{g_t}^{1-\holde} - \frac{4(L_1 + L_2)\lambda^\holde}{1-\gamma} \right).
    \end{align*}
   Consequently, we can divide our analysis into the following cases: either $\norm{g_t}^{1-\holde} \geq \frac{16(L_1+L_2)\lambda^\holde}{1-\gamma}$ or $\norm{g_t}^{1-\holde} <  \frac{16(L_1+L_2)\lambda^\holde}{1-\gamma}$. By analogy to the previous analysis, in the first case we get
   \begin{align*}
        &\lambda \norm{g_t}^{\holde+1} \left( \frac{1}{2} \norm{g_t}^{1-\holde} - \frac{4(L_1 + L_2)\lambda^\holde}{1-\gamma} \right)\\ &\geq \frac{\lambda}{4}  \norm{g_t}^2 \\
   \end{align*}i
   
   \begin{align*}
        &\lambda \norm{g_t}^{\holde+1} \left( \frac{1}{2} \norm{g_t}^{1-\holde} - \frac{4(L_1 + L_2)\lambda^\holde}{1-\gamma} \right)\\ &\geq \frac{\lambda}{2}  \norm{g_t}^2 - \left(\frac{16(L_1 + L_2)\lambda^\holde}{1-\gamma} \right)^{\frac{\holde+1}{1-\holde}} \left(\frac{4(L_1 + L_2)\lambda^{\holde+1}}{1-\gamma} \right) \\
        &\geq \frac{\lambda}{2} \norm{g_t}^{2} - 4^{\frac{3+\holde}{1-\holde}} \left(\frac{L_1+L_2}{1-\gamma}\right)^{\frac{2}{1-\holde}} \lambda^{\frac{\holde+1}{1-\holde}},
   \end{align*}
   with a resulting bound
   \begin{align*}
        \lambda \norm{g_t}^{\holde+1} \left( \frac{1}{2} \norm{g_t}^{1-\holde} - \frac{4(L_1 + L_2)\lambda^\holde}{1-\gamma} \right) \geq \frac{\lambda}{4} \norm{g_t}^{2} - 4^{\frac{3+\holde}{1-\holde}} \left(\frac{L_1+L_2}{1-\gamma}\right)^{\frac{2}{1-\holde}} \lambda^{\frac{\holde+1}{1-\holde}},
   \end{align*}
   which holds in both cases. This yields for the final bound if $B \geq \frac{\sigma^2}{(1-\gamma)^2}$
    \begin{align*}
        \frac{1}{T} \sum_{t=1}^T \Exp \sq{\norm{g_t}^2} &\leq \frac{4}{\lambda T} \left( J_* - J(\theta_0)\right)+ \left(\frac{\sigma}{(1-\gamma) \sqrt{B}}\right)^{\holde+1}+ 2^{\frac{8}{1-\holde}} \left(\frac{L_1+L_2}{1-\gamma}\right)^{\frac{2}{1-\holde}} \lambda^{\frac{2\holde}{1-\holde}}.
    \end{align*}
    This results in a bias term that cannot be removed by minibatching, but can be shrunk with an appropriate choice of $\lambda$ small.
    
    \textbf{Learning Rate III:} We can generalize our analysis to a decaying learning rate $h_t = \lambda t^{-q}$ for $q \in [0,1)$.  We again follow the same steps as for the previous learning rate, only briefly sketching the argument below. Again we derive under the same condition $\lambda^{-\holde} \geq \frac{2(L_1 + L_2)}{1-\gamma}$ that
    \begin{align*}
        & h_t \norm{g_t}^2 - \frac{2(L_1+L_2)}{1-\gamma} \left(h_t^{\frac{\beta_1}{4} +1}\norm{g_t}^{\frac{\beta_1}{4}+1} + h_t^{\beta_2 + 1}\norm{ g_t}^{\beta_2 + 1}\right) \\
        &\qquad \geq \lambda t^{-q} \norm{g_t}^{\holde+1} \left( \frac{1}{2} \norm{g_t}^{1-\holde} - \frac{4(L_1 + L_2)\lambda^\holde}{1-\gamma} t^{-q\holde}\right).
    \end{align*}
    Dividing again into the two cases: either $\norm{g_t}^{1-\holde} \geq \frac{16(L_1+L_2)\lambda^\holde}{1-\gamma} t^{-q\holde}$ or $\norm{g_t}^{1-\holde} <  \frac{16(L_1+L_2)\lambda^\holde}{1-\gamma} t^{-q\holde}$, with a resulting bound of
   \begin{align*}
        \lambda t^{-q} \norm{g_t}^{\holde+1} \left( \frac{1}{2} \norm{g_t}^{1-\holde} - \frac{4(L_1 + L_2)\lambda^\holde}{1-\gamma} t^{-q\holde}\right) \geq \frac{\lambda}{4} t^{-q} \norm{g_t}^{2} - \frac{\lambda}{4} t^{\frac{-q\holde - q}{1-\holde}},
   \end{align*}
   which holds in both cases if $\lambda^{-\holde} \geq \frac{16(L_1 + L_2)}{1-\gamma}$. Consequently, noting that $\sum_{t=1}^T t^{-p} \leq \frac{T^{1-p}}{1-p}$ for $p \in [0,1)$ and using  $B \geq \sigma^2 (1-\gamma)^{-2}$
    \begin{align*}
        \frac{1}{T} \sum_{t=1}^T \Exp \sq{\norm{g_t}^2} &\leq \frac{4}{\lambda T^{1-q}} \left( J_* - J(\theta_0)\right) +\frac{1-\holde}{1-2q\holde-\holde} T^{\frac{-2q\holde}{1-\holde}} \\
        &\qquad +\frac{1}{1-q\holde} T^{-q\holde} \left(\frac{\sigma}{(1-\gamma) \sqrt{B}}\right)^{\holde+1},
    \end{align*}
    as long as the exponent $q < \frac{1-\holde}{2\holde}$. If instead $q \geq \frac{1-\holde}{2\holde}$, we use simply $\sum_{t=1}^T t^p \leq 4\ln T$ for $p \leq 1$ to get
    \begin{align*}
        \frac{1}{T} \sum_{t=1}^T \Exp \sq{\norm{g_t}^2} &\leq \frac{4 }{\lambda T^{1-q}}\left( J_* - J(\theta_0)\right) + \frac{4 \ln T}{T} +\frac{1}{1-q\holde} T^{-q\holde} \left(\frac{\sigma}{(1-\gamma) \sqrt{B}}\right)^{\holde+1}.
    \end{align*}
    Note here that the first term is dominant for all $q <1$.
    
    \hfill $\square$

    We can also note that $\min_{t \leq T} \Exp\sq{\norm{g_t}^p} \leq \frac{1}{T} \sum_{t \leq T} \Exp\sq{\norm{g_t}^p}$ for any power $p$, which allows us to bound the minimum gradient for $t \leq T$ as well.

    \subsection{Natural Policy Gradient}
    
    \textbf{Facts:} Assumption \ref{as:moment} also guarantees that the Fisher information matrix $K(\theta) = \Exp \sq{\psi_\theta(s,a) \psi_\theta(s,a)^T}$ has bounded norm, since by the trace inequality
    \begin{align*}
        \norm{K(\theta)} &= \norm{\int \int \psi_\theta(s,a) \psi_\theta(s,a)^\top d_{\theta}^\rho(s,a) \, da \, ds} \\
        &\leq \int \int \norm{\psi_\theta(s,a) \psi_\theta(s,a)^\top} d_{\theta}^\rho(s,a) \, da \, ds \\
        &\overset{}{\leq} \int \int \norm{\psi_\theta(s,a)}^2 d_{\theta}^\rho(s,a)\, da \, ds \leq \psi_{\infty}.
    \end{align*}
    This will be important when analyzing the NPG algorithm. Likewise, by the triangle inequality and the i.i.d. sampling,
    \begin{align*}
        \Exp \sq{\norm{K_t}|\theta_{t-1}} &\overset{}{=} \Exp \sq{\norm{\frac{1}{B} \sum_{i=1}^B  \psi_{\theta_{t-1}}(s'_{t,i},a'_{t,i}) \psi_{\theta_{t-1}}(s'_{t,i},a'_{t,i})^\top} \Bigg|\theta_{t-1}} \\
        &\leq \frac{1}{B} \sum_{i=1}^B \Exp_{s, a \sim d_{\theta_{t-1}}^\rho}\sq{  \norm{\psi_{\theta_{t-1}}(s,a) \psi_{\theta_{t-1}}(s,a)^\top}} \\
        &\leq \psi_{\infty}.
    \end{align*}

    \noindent \textit{Proof for Theorem \ref{th:con}, NPG}:
    We now perform the same analysis for the exact natural policy gradient, applying mean value theorem and Proposition \ref{prop:smooth} to get
    \begin{align*}
            J(\theta_t) &\geq J(\theta_{t-1}) +\inner{g_t}{h_t (K_t +   \xi I)^{-1} 
            (g_t+e_t)} - \frac{L_1}{1-\gamma} \norm{h_t (K_t + \xi I)^{-1}  (g_t+e_t)}^{\frac{\beta_1}{4}+1} \\
            &\qquad -  \frac{L_2}{1-\gamma}\norm{h_t (K_t + \xi I)^{-1}  (g_t+e_t)}^{\beta_2+1} \\
            &= J(\theta_{t-1}) +\inner{g_t}{h_t (K_t + \xi I)^{-1} 
            g_t} + \inner{g_t}{h_t (K_t + \xi I)^{-1} e_t} \\
            &\qquad - \frac{L_1}{1-\gamma} \norm{h_t (K_t + \xi I)^{-1}  (g_t+e_t)}^{\frac{\beta_1}{4}+1} -  \frac{L_2}{1-\gamma}\norm{h_t (K_t + \xi I)^{-1}  (g_t+e_t)}^{\beta_2+1} \\
            &{\geq} J(\theta_{t-1}) + \inner{g_t}{h_t (K_t + \xi I)^{-1}g_t}  + \inner{g_t}{h_t (K_t + \xi I)^{-1} e_t} \\
            &\qquad -  \frac{L_1}{(1-\gamma)\xi^{\frac{\beta_1}{4}+1}} h_t^{\frac{\beta_1}{4}+1} \norm{g_t+e_t}^{\frac{\beta_1}{4}+1} - \frac{L_2}{(1-\gamma)\xi^{\beta_2+1}} h_t^{\beta_2+1} \norm{g_t+e_t}^{\beta_2+1} \\
            &{\geq} J(\theta_{t-1}) + \inner{g_t}{h_t (K_t + \xi I)^{-1}g_t}  + \inner{g_t}{h_t (K_t + \xi I)^{-1} e_t} \\
            &\qquad -  \frac{2L_1}{(1-\gamma)\xi^{\frac{\beta_1}{4}+1}} h_t^{\frac{\beta_1}{4}+1} \left(\norm{g_t}^{\frac{\beta_1}{4}+1} + \norm{e_t}^{\frac{\beta_1}{4}+1} \right) - \frac{2L_2}{(1-\gamma)\xi^{\beta_2+1}} h_t^{\beta_2+1} \left(\norm{g_t}^{\beta_2+1} + \norm{e_t}^{\beta_2+1} \right).
    \end{align*}
    This is entirely analogous to the standard policy gradient case, except for the terms $(K(\theta_{t-1}) + \xi I)$ which appears from the natural policy gradient.
    
    Finally, we take expectation. First,
    \begin{align*}
        \Exp \sq{\inner{g_t}{h_t (K_t + \xi I)^{-1}g_t}} &= \Exp \sq{\inner{g_t}{h_t (K_t + \xi I)^{-1} g_t}} \\
        &\overset{(i)}{\geq} \Exp \sq{ \Exp\sq{\frac{1}{\norm{K_t} + \xi}\Big|\theta_{t-1}} \inner{g_t}{h_t g_t}} \\
        &\overset{(ii)}{\geq} h_t \Exp \sq{ \Exp\sq{\norm{K_t} + \xi\Big|\theta_{t-1}}^{-1} \norm{g_t}^2} \\
        &\overset{(iii)}{\geq} h_t (\psi_\infty + \xi)^{-1} \Exp \sq{\norm{g_t}^2},
        \end{align*}
    where in $(i)$ we use that the minimum eigenvalue of $(K_t+\xi I)^{-1}$ is lower bounded by $\frac{1}{\norm{K_t} + \xi}$. In $(ii)$ we use the fact that for any positive random variable $X$, $\Exp \sq{1/X} \geq \Exp \sq{X}^{-1}$ by Jensen's inequality, and in $(iii)$ we use our bound on $\Exp \sq{\norm{K_t}}$.
    
    Secondly, we note the following
    \begin{align*}
        \Exp \sq{h_t \inner{g_t}{h_t (K_t + \xi I)^{-1} e_t}} &=  \Exp \sq{ \inner{g_t}{h_t \Exp \sq{(K_t + \xi)^{-1} e_t \Big|\theta_{t-1}}} } = 0,
    \end{align*}
    since $K_t$ is independent of $e_t$ given $\theta_{t-1}$ from our sampling procedure, and $e_t$ has zero expectation given $\theta_{t-1}$, where again the inner expectation is on the randomness in the sampling procedure at iteration $t$, and the outer expectation is with respect to the random variable $\theta_{t-1}$.
    
    Finally, this yields for $\xi < 1$
    \begin{align*}
        h_t \Exp \sq{\norm{g_t}^2} &\leq (\psi_\infty + \xi) \Exp\sq{J(\theta_t) - J(\theta_{t-1})} + \frac{2L_1(\psi_\infty + \xi)}{(1-\gamma)\xi^{\maxholde+1}} h_t^{\frac{\beta_1}{4} +1}\left(\Exp \sq{\norm{g_t}^{\frac{\beta_1}{4}+1}} + \left(\frac{\sigma}{(1-\gamma)\sqrt{B}} \right)^{\frac{\beta_1}{4}+1} \right) \\
        &\qquad + \frac{2L_2(\psi_\infty + \xi)}{(1-\gamma)\xi^{\maxholde+1}} h_t^{\beta_2 + 1}\left(\Exp\sq{\norm{ g_t}^{\beta_2 + 1}} + \left(\frac{\sigma}{(1-\gamma)\sqrt{B}} \right)^{\beta_2 + 1} \right).
    \end{align*}
    Summing from $t = 1 \ldots T$, and bounding $J(\theta_t)$ with $J_*$ completes the proof, if we take, $$C_{NPG,1} = (\psi_\infty+\xi), \qquad C_{NPG,2} = \frac{2(L_1+L_2)(\psi_\infty + \xi)}{\xi^{\maxholde+1}}.$$ \hfill $\square$
    
    \noindent \textit{Proof of Corollary \ref{cor:convergence}, NPG:} 
    
    The proof techniques for generating the convergence rates are nearly identical to policy gradient, except for the change in the constant term $C_{NPG}$. We show the constant learning rate as an illustrative case:
    
    \textbf{Learning Rate I:}
    Let $\lambda \leq 1$, then by Jensen's inequality and that $h_t \leq 1$
    \begin{align*}
        h_t \Exp \sq{\norm{g_t}^2} - &\frac{2(L_1+L_2)(\psi_\infty + \xi)}{(1-\gamma)\xi^{\maxholde + 1}} \left(h_t^{\frac{\beta_1}{4} +1}\Exp \sq{\norm{g_t}^{\frac{\beta_1}{4}+1}} + h_t^{\beta_2 + 1} \Exp \sq{\norm{ g_t}^{\beta_2 + 1}}\right) \\
        &\leq h_t \Exp \sq{\norm{g_t}^2} - \frac{2(L_1+L_2)(\psi_\infty + \xi)}{(1-\gamma)\xi^{\maxholde + 1}} h_t^{\holde+1}\left(\Exp \sq{\norm{g_t}^{2}}^{\frac{\beta_1 + 4}{8}} + \Exp \sq{\norm{ g_t}^2}^{\frac{\beta_2 + 1}{2}}\right) \\
        &= \lambda \Exp \sq{\norm{g_t}^2}^{\frac{\holde+1}{2}} \left( \Exp \sq{\norm{g_t}^2}^{\frac{1-\holde}{2}} - \frac{2(L_1+L_2)(\psi_\infty + \xi)}{(1-\gamma)\xi^{\maxholde + 1}} \lambda^\holde\left(1+ \Exp \sq{\norm{ g_t}^2}^{\frac{\maxholde - \holde}{2}}\right)\right).
    \end{align*}
    Now either $\Exp \sq{\norm{ g_t}^2} \geq 1$, in which case if $\lambda^{-\holde} \geq \frac{4(L_1+L_2)(\psi_\infty + \xi)}{(1-\gamma)\xi^{\maxholde + 1}}$, then we can bound this by
    \begin{align*}
        \lambda \Exp \sq{\norm{g_t}^2}^{\frac{\holde+1}{2}} \left( \frac{1}{2}\Exp \sq{\norm{g_t}^2}^{\frac{1-\holde}{2}} - \frac{2(L_1+L_2)(\psi_\infty + \xi)}{(1-\gamma)\xi^{\maxholde + 1}} \lambda^\holde\right),
    \end{align*}
    or $\Exp \sq{\norm{ g_t}^2} < 1$, in which case the following bound suffices:
    \begin{align*}
        \lambda \Exp \sq{\norm{g_t}^2}^{\frac{\holde+1}{2}} \left(\Exp \sq{\norm{g_t}^2}^{\frac{1-\holde}{2}} - \frac{4(L_1+L_2)(\psi_\infty + \xi)}{(1-\gamma)\xi^{\maxholde + 1}} \lambda^\holde\right).
    \end{align*}
    In both cases, we can lower bound the quantity with
    \begin{align*}
        \lambda \Exp \sq{\norm{g_t}^2}^{\frac{\holde+1}{2}} \left(\frac{1}{2} \Exp \sq{\norm{g_t}^2}^{\frac{1-\holde}{2}} - \frac{4(L_1+L_2)(\psi_\infty + \xi)}{(1-\gamma)\xi^{\maxholde + 1}} \lambda^\holde\right).
    \end{align*}
    
    Now again we split the analysis into two cases. If $\Exp \sq{\norm{g_t}^2}^{\frac{1-\holde}{2}} \geq \frac{16(L_1+L_2)(\psi_\infty + \xi)}{(1-\gamma)\xi^{\maxholde + 1}} \lambda^\holde$, in which case the following bound suffices:
    \begin{align*}
        \lambda \Exp \sq{\norm{g_t}^2}^{\frac{\holde+1}{2}} &\left(\frac{1}{2} \Exp \sq{\norm{g_t}^2}^{\frac{1-\holde}{2}} - \frac{1}{4} \Exp \sq{\norm{g_t}^2}^{\frac{1-\holde}{2}}\right) \\
        &\geq \frac{\lambda}{2} \lambda \Exp \sq{\norm{g_t}^2}. 
    \end{align*}
    In the other case, when $\Exp \sq{\norm{g_t}^2}^{\frac{1-\holde}{2}} < \frac{16(L_1+L_2)(\psi_\infty + \xi)}{(1-\gamma)\xi^{\maxholde + 1}} \lambda^\holde$, we get
    \begin{align*}
        \lambda \Exp \sq{\norm{g_t}^2}^{\frac{\holde+1}{2}} &\left(\frac{1}{2} \Exp \sq{\norm{g_t}^2}^{\frac{1-\holde}{2}} - \frac{4(L_1+L_2)(\psi_\infty + \xi)}{(1-\gamma)\xi^{\maxholde + 1}} \lambda^\holde\right) \\
        &= \frac{\lambda}{2} \Exp \sq{\norm{g_t}^2} -  \left(\frac{4(L_1+L_2)(\psi_\infty + \xi)}{(1-\gamma)\xi^{\maxholde + 1}} \lambda^{\holde+1} \right) \left(\frac{16(L_1+L_2)(\psi_\infty + \xi)}{(1-\gamma)\xi^{\maxholde + 1}} \lambda^\holde \right)^{\frac{\holde+1}{1-\holde}} \\
        &\overset{(i)}{\geq}\frac{\lambda}{2} \Exp \sq{\norm{g_t}^2} - 4^{\frac{3+\holde}{1-\holde}} \left(\frac{(L_1+L_2)(\psi_\infty + \xi)}{(1-\gamma)\xi^{\maxholde+1}}\right)^{\frac{2}{1-\holde}} \lambda^{\frac{\holde+1}{1-\holde}}.
    \end{align*}
    In either case the worst case bound $$\frac{\lambda}{4} \Exp \sq{\norm{g_t}^2} -  4^{\frac{3+\holde}{1-\holde}} \left(\frac{(L_1+L_2)(\psi_\infty + \xi)}{(1-\gamma)\xi^{\maxholde+1}}\right)^{\frac{2}{1-\holde}} \lambda^{\frac{\holde+1}{1-\holde}}$$ suffices. 
    
    Plugging this into our earlier expression for $B \geq \sigma^2 (1-\gamma)^{-2}$ and rearranging, using again that $ \lambda^{-\holde} \geq \frac{4(L_1+L_2)(\psi_\infty + \xi)}{(1-\gamma)\xi^{\maxholde + 1}}$
    \begin{align*}
        \frac{\lambda}{4} \Exp \sq{\norm{g_t}^2} &\leq  (\psi_\infty + \xi)  \Exp\sq{J(\theta_t) - J(\theta_{t-1})} + 4^{\frac{3+\holde}{1-\holde}} \left(\frac{(L_1+L_2)(\psi_\infty + \xi)}{(1-\gamma)\xi^{\maxholde+1}}\right)^{\frac{2}{1-\holde}} \lambda^{\frac{\holde+1}{1-\holde}} \\
        &\qquad + \frac{4(L_1+L_2)(\psi_\infty + \xi)}{(1-\gamma)\xi^{\maxholde+1}} \lambda^{\beta_0 +1}\left(\frac{\sigma}{(1-\gamma)\sqrt{B}} \right)^{\frac{\beta_1}{4}+1} \\
        &\leq (\psi_\infty + \xi)  \Exp\sq{J(\theta_t) - J(\theta_{t-1})} + 4^{\frac{3+\holde}{1-\holde}} \left(\frac{(L_1+L_2)(\psi_\infty + \xi)}{(1-\gamma)\xi^{\maxholde+1}}\right)^{\frac{2}{1-\holde}} \lambda^{\frac{\holde+1}{1-\holde}}\\
        &\qquad + \frac{\lambda}{4}\left(\frac{\sigma}{(1-\gamma)\sqrt{B}} \right)^{\frac{\beta_1}{4}+1}.
    \end{align*}
    It remains to average from $1 \ldots T$ and use the supremum bound on $J$, to get
    \begin{align*}
        \frac{1}{T} \sum_{t=1}^T \Exp \sq{\norm{g_t}^2} &\leq (\psi_\infty + \xi) (J_* - J(\theta_0)) + 2^{\frac{8}{1-\holde}}\left(\frac{(L_1+L_2)(\psi_\infty + \xi)}{(1-\gamma)\xi^{\maxholde+1}}\right)^{\frac{2}{1-\holde}} \lambda^{\frac{2\holde}{1-\holde}} \\
        &\qquad + \left(\frac{\sigma}{(1-\gamma)\sqrt{B}} \right)^{\frac{\beta_1}{4}+1}.
    \end{align*}

    As can be seen from this analysis, the only difference in our analysis from the same learning rate in the policy gradient is that we manipulate the expectations of the quantities, and the slight difference in constants. The convergence rates for the other two learning rates follow from analogy to the policy gradient case.
    \hfill $\square$

\section{Theorem \ref{th:opt}}

 As with Theorem \ref{th:con}, we divide our analyze between policy gradient and natural policy gradient. We will use the same notation as in the proof of Theorem \ref{th:con}.

    \subsection{Policy Gradient}
     Note that for standard policy gradient, the stationary points can in general be highly sub-optimal, see \cite{agarwal2020optimality}. Under a strong first order condition, however, we can obtain some guarantees.

    Consequently, we can use Assumption \ref{as:glob} to show the global optimality of our algorithm for vanilla policy gradient.
    
    \noindent \textit{Proof of Theorem \ref{th:opt}, PG:} This follows directly from Assumption \ref{as:glob}, as we find that for any $\theta_*$ that maximizes $J$:
    \begin{align*}
        \min_{t=1, \ldots T} J(\theta_*) - J(\theta_{t-1}) &\leq \min_{t=1, \ldots T} \frac{m\norm{\theta_{t-1} - \theta_*}}{1-\gamma} \norm{g_t} \\
        &\leq \frac{m \text{Diam}(\Theta)}{(1-\disc)} \min_{t=1, \ldots T} \norm{g_t}, 
    \end{align*}
    since $g_t = \nabla J(\theta_{t-1})$.
    \hfill $\square$
    
    To show the corollary, it remains to substitute the bounds for $\min_{t=1, \ldots T} \norm{g_t} $ that we obtain from Corollary \ref{cor:convergence}. When converted to a sample complexity bound, this yields the final result. \\
    
    \noindent \textit{Proof of Corollary \ref{cor:opt}, PG}:
    
    It remains now to substitute the learning rate and gradient bounds from each of the cases we analyzed earlier; doing so obtains the rates found in the corollary. In the sequel, $\lesssim$ and $\gtrsim$ only concern terms depending on $\lambda, \epsilon, \gamma$.
    
    \textbf{Case I, $h_t = \lambda$}: Recall that in this case, 
    \begin{align*}
    \min_{t \leq T} \norm{g_t}^2 \leq\frac{4}{\lambda T} \left( J_* - J(\theta_0)\right)+ \left(\frac{\sigma}{(1-\gamma) \sqrt{B}}\right)^{\holde+1}+ 2^{\frac{8}{1-\holde}} \left(\frac{L_1+L_2}{1-\gamma}\right)^{\frac{2}{1-\holde}} \lambda^{\frac{2\holde}{1-\holde}}.
    \end{align*}
    Subsequently, we can choose $\lambda \lesssim (1-\gamma)^{\frac{2-\holde}{\holde}} \epsilon^{\frac{1-\holde}{\holde}}$, so that the final term is $\lesssim \epsilon^2 (1-\gamma)^2$. Note that this choice does not conflict with our earlier constraints on $\lambda$ from Corollary \ref{cor:convergence}.
    
    Subsequently we require $T \gtrsim \frac{1}{\lambda \epsilon^2 (1-\gamma)^2}$ to make the first term $\lesssim (1-\gamma)^2\epsilon^2$  and $B \gtrsim \epsilon^{-\frac{4}{\holde+1}} (1-\gamma)^{-2 - \frac{4}{\holde+1}}$ to make the second term $\lesssim (1-\gamma)^2\epsilon^2$. 
   Substituting our requirement on $\lambda$, we get $T \gtrsim \epsilon^{-\frac{1+\holde}{\holde}} (1-\gamma)^{-\frac{2+\holde}{\holde}}.$ Thus, we get $\min_{t \leq T} \norm{g_t}^2 \lesssim (1-\gamma)^2\epsilon^2$. This then implies
    \begin{align*}
        \min_{t \leq T} J(\theta_*) - \Exp \sq{J(\theta_{t-1})} \lesssim \epsilon.
    \end{align*}
    
    \textbf{Case II, $h_t = \lambda T^{\frac{\holde-1}{\holde+1}}$}:  Recall that in this case, 
    $$\min_{t \leq T} \norm{g_t}^2 \leq \frac{4}{\lambda} T^{-\frac{2\holde}{\holde+1}} \left(J_* - J(\theta_0) + \frac{\lambda}{4} \right) + T^{\frac{\holde(\holde-1)}{\holde+1}} \left(\frac{\sigma}{(1-\gamma) \sqrt{B}}\right)^{\holde+1}.$$
    Subsequently we require $T \gtrsim \left(\frac{1}{\lambda \epsilon^2 (1-\gamma)^2}\right)^{\frac{\holde+1}{2\holde}}$ to bound the first term by $\lesssim \epsilon^{2} (1-\gamma)^{2}$. For simplicity, take $\lambda \asymp (1-\gamma)^{1/\holde}$, in which case we get $T \gtrsim \left(\frac{1}{\epsilon^2 (1-\gamma)^{2+{1/\holde}}}\right)^{\frac{\holde+1}{2\holde}}$.
    Consequently, the second term is bounded by
    \begin{align*}
        T^{\frac{\holde(\holde-1)}{\holde+1}} \left(\frac{\sigma}{(1-\gamma) \sqrt{B}}\right)^{\holde+1} &\lesssim \left(\frac{1}{\epsilon^2 (1-\gamma)^{2+1/\holde}}\right)^{\frac{\holde-1}{2}} \left(\frac{1}{(1-\gamma) \sqrt{B}}\right)^{\holde+1} \\
        &\lesssim \epsilon^{1-\holde} (1-\gamma)^{-\frac{4\holde^2+\holde-1}{2\holde}}B^{-\frac{\holde+1}{2}}.
    \end{align*}
    So we can take $B \gtrsim \epsilon^{-2} (1-\gamma)^{\frac{4\holde^2 + 5\holde - 1}{\holde(\holde+1)}}$,  to make this term $\lesssim (1-\gamma)^2\epsilon^2$. This then implies
    \begin{align*}
        \min_{t\leq T} J(\theta_*)  - \Exp \sq{J(\theta_{t-1})} \lesssim \epsilon.
    \end{align*}
    
    \hfill $\square$

    \subsection{Natural Policy Gradient}
    
    First we show the following lemma:
    \begin{lemma}\label{lem:diff_fisher}
        For any $\theta \in \Theta$ and vector $x$, the following holds:
        \begin{equation*}
            \norm{(K(\theta_{t-1}) + \xi I)^{-1} x - K(\theta_{t-1})^\dagger x } \leq C_{\xi} \norm{x},
        \end{equation*}
        where $C_{\xi} = \max(\zeta^{-1}, \xi^{-1}) > 0$ and $\zeta = \inf_{\theta} \min_{k: \lambda_k > 0} \lambda_k(K(\theta)) \geq 0$ is a policy-dependent regularity constant, which measures the smallest non-zero eigenvalue of $K(\theta)$ across all policies $\theta$.
    \end{lemma}
    
    \textbf{Remarks:} We adapt this from a similar result presented in \cite{xu2020improving}. This lemma is used to bound the stable inverse Fisher matrix $(K(\theta_t) + \xi I)^{-1}$ from the unstable pseudo-inverse $K(\theta_t)^{\dagger}$. Not only is the computation of the inverse far simpler than computation of the pseudo-inverse, but it will also be used to upper bound the rate of convergence. We can rewrite the condition number $\zeta$ in terms of the condition number used in \cite{agarwal2020optimality}, but this introduces unnecessary complexity for our purposes.
    
    \begin{proof} Since $K(\theta)$ is real and symmetric, it is diagonalizable into $U(\theta)^\top \Xi(\theta) U(\theta)$, where $U(\theta) = [u_1, u_2, \ldots u_d]$ is the orthonormal space of eigenvectors (and appropriate basis for the null-space) and $\Xi(\theta) = \text{diag}(\xi_1, \ldots, \xi_d)$ where $\xi_i$ is the $i$-th eigenvalue (including zeroes). Thus we can write the following:
    \begin{equation*}
        \begin{split}
            \sq{K(\theta) + \xi I}^{-1} x &= U(\theta)^\top (\Xi(\theta)^{-1} + \frac{1}{\xi} I) U(\theta) x, \\
            &= U(\theta)^\top \sq{\frac{1}{\xi_1 + \xi} \inner{u_1}{x}, \frac{1}{\xi_2 + \xi}  \inner{u_2}{x}\ldots \frac{1}{\xi_d + \xi}  \inner{u_d}{x}}^\top. 
        \end{split}
    \end{equation*}
    Analogously, the exact pseudo-inverse yields:
    \begin{equation*}
    \begin{split}
            K(\theta)^\dagger x &= U(\theta)^\top \Xi(\theta)^{\dagger} U(\theta) x \\
            &= U(\theta)^\top \sq{\frac{1}{\xi_1} \inner{u_1}{x}, \frac{1}{\xi_2}  \inner{u_2}{x}\ldots \frac{1}{\xi_k}  \inner{u_k}{x}, 0 \ldots 0}^\top, 
        \end{split}
    \end{equation*}
    where $k$ is the rank of $K(\theta)$. Subsequently their difference is:
    \begin{equation*}
    \begin{split}
        & K(\theta + \xi I)^{-1} x - K(\theta)^\dagger x \\
        &\leq U(\theta)^\top \sq{\left(\frac{1}{\xi_1 + \xi} - \frac{1}{\xi_1}\right) \inner{u_1}{x}, \ldots \left(\frac{1}{\xi_k + \xi} - \frac{1}{\xi_k}\right)  \inner{u_k}{x}, \frac{1}{\xi} \inner{u_{k+1}}{x} \ldots \frac{1}{\xi} \inner{u_d}{x}}^\top \\
        &= - \xi U(\theta)^\top \sq{\left(\frac{1}{\xi_1(\xi_1 + \xi)}\right) \inner{u_1}{x}, \ldots \left(\frac{1}{\xi_k(\xi_k + \xi)}\right)  \inner{u_k}{x}, -\frac{1}{\xi^2} \inner{u_{k+1}}{x} \ldots -\frac{1}{\xi^2} \inner{u_d}{x}}^\top.
    \end{split}    
    \end{equation*}
    Since the information matrix is always positive semidefinite, we bound $\xi_1 + \xi \geq \xi_k + \xi \geq \xi$. Thus the norm is bounded:
    \begin{equation*}
        \begin{split}
        \norm{K(\theta + \xi I)^{-1} x - K(\theta)^\dagger x} &\leq \xi \norm{U(\theta)^\top} \max(\xi^{-2}, \xi^{-1} \zeta^{-1}) \norm{x}\\
        & \leq \max(\zeta^{-1}, \xi^{-1}) \norm{x} \triangleq C_{\xi} \norm{x},
        \end{split}
    \end{equation*}
    where $\zeta$ is described in the lemma statement, and we use the fact that $U$ is unitary. \end{proof}
    
    We adapt another Lemma without proof, to quantify the gap between any two policies in a policy class.
    \begin{lemma} \label{lem:performance_diff}
        (Performance Difference Lemma, Adapted from \cite[Lemma 3.2]{agarwal2020optimality}) For all policies $\pi_{\theta_1}, \pi_{\theta_2}$ and all initial distributions $\rho$:
        \begin{equation*}
            J(\theta_1) - J(\theta_2) \triangleq \Exp_{s \sim \rho}\sq{V_{\theta_1}(s) - V_{\theta_2}(s)} = \frac{1}{1-\disc} \Exp_{s,a \sim d_{\theta_1}^{\rho}} \sq{A_{\theta_2}(s,a)}.
        \end{equation*}
    \end{lemma}
    
    \noindent \textit{Proof of Theorem \ref{th:opt} for NPG:} Consider a $\mathsf{KL}$ function on the policies at each state: $F(\theta, s) = \mathsf{KL}(\pi_*(\cdot|s), \pi_{\theta}(\cdot|s)) \triangleq \int_\A \pi_*(a| s) \log \frac{\pi_*(a| s)}{\pi_{\theta}(a| s)} \, da$, where $\pi_* \triangleq \pi_{\theta_*}$ is any policy that maximizes $J$. Let $d_*^\rho $ be the visitation distribution measure induced by $\pi_*$ on $\State \times \A$, and let $H_*^\rho$ be the state component of this distribution measure as in Appendix A. If we define $F(\theta) = \Exp_{s \sim H_*^\rho} \sq{F(\theta, s)}$, then
    \begin{align*}
        F(\theta_t) - F(\theta_{t-1}) &\triangleq \Exp_{s \sim H_*^\rho} \sq{F(\theta_t, s) - F(\theta_{t-1}, s)} \\
        &= \Exp_{s \sim H_*^\rho} \sq{\int_{\A} \left(\log \frac{\pi_*(a| s)}{\pi_{\theta_t}(a| s)} - \log \frac{\pi_*(a| s)}{\pi_{\theta_{t-1}}(a| s)}\right)\pi_*(a| s) \, da } \\
            &= \Exp_{s, a \sim d_{*}^\rho} \sq{\log \pi_{\theta_{t-1}}(a|s) - \log \pi_{\theta_t}(a|s)} \\
            &\overset{(i)}{\geq}
            \Exp_{s, a \sim d_{*}^\rho} \sq{\nabla \log \pi_{\theta_{t-1}}(a|s)}^\top (\theta_t - \theta_{t-1}) - C_{\nu,2} \norm{\theta_t - \theta_{t-1}}^{\beta_2+1} \\
            &\triangleq \Exp_{s, a \sim d_{*}^\rho} \sq{\psi_{\theta_{t-1}}(s, a)}^\top (\theta_t - \theta_{t-1}) - C_{\nu,2} \norm{\theta_t - \theta_{t-1}}^{\beta_2+1},
    \end{align*}
    where in $(i)$ we use the Hölder-smoothness condition on $\psi_{\theta} \triangleq \nabla \log \pi_\theta$ from Assumption \ref{as:smooth}. Let $g_t + e_t$ be the exact and error components of the approximate gradient as before. Concentrating on the first term, we can make use of the triangle inequality to obtain the following:
    \begin{align*}
        &\Exp_{s, a \sim d_{*}^\rho} \sq{\psi_{\theta_{t-1}}(s, a)}^\top (\theta_t - \theta_{t-1}) \\
        &= \Exp_{s, a \sim d_{*}^\rho} \sq{\psi_{\theta_{t-1}}(s, a)}^\top (h_t (K_t + \xi I)^{-1} (g_t+e_t)) \\
        &= \underbrace{\Exp_{s, a \sim d_{*}^\rho} \sq{\psi_{\theta_{t-1}}(s, a)}^\top (h_t (K_t + \xi I)^{-1} g_t - h_t (K(\theta_{t-1}) + \xi I)^{-1} g_t)}_{(a)} \\
        & \quad + \underbrace{\Exp_{s, a \sim d_{*}^\rho} \sq{\psi_{\theta_{t-1}}(s, a)}^\top (h_t (K(\theta_{t-1}) + \xi I)^{-1} g_t - h_t K(\theta_{t-1})^\dagger g_t)}_{(b)} \\
        & \quad + \underbrace{\Exp_{s, a \sim d_{*}^\rho} \sq{\psi_{\theta_{t-1}}(s,a)^\top h_t K(\theta_{t-1})^\dagger g_t - h_t A_{\theta_{t-1}}(s,a)}}_{(c)} + \underbrace{\Exp_{s, a \sim d_{*}^\rho} \sq{h_t A_{\theta_{t-1}}(s,a)}}_{(d)} \\
        &\quad + \underbrace{\Exp_{s, a \sim d_{*}^\rho} \sq{\psi_{\theta_{t-1}}(s, a)}^\top (h_t (K_t + \xi I)^{-1} e_t)}_{(e)}.
    \end{align*}
    Briefly explaining the meaning of each of these: \textbf{(a)} represents the difference between using the sampled Fisher matrix $K_t$ and the true matrix $K(\theta_{t-1})$, while \textbf{(b)} represents the difference between the stabilized matrix $(K+\xi I)^{-1}$ and the unstabilized matrix $K^\dagger$, \textbf{(c)} is used to represent the difference between the natural gradient objective and the true function, \textbf{(d)} represents the true objective, and \textbf{(e)} represents the error due to the approximate gradient.
    
    Consequently we bound each of these terms individually. We begin with $(b)$:
    \begin{align*}
            (b): \quad \quad &\Exp_{s, a \sim d_{*}^\rho} \sq{\psi_{\theta_{t-1}}(s, a)}^\top (h_t (K(\theta_{t-1}) + \xi I)^{-1} g_t - h_t (K(\theta_{t-1})^\dagger g_t) \\
            &\overset{}{\geq} -h_t\norm{\Exp_{s, a \sim d_{*}^\rho} \sq{\psi_{\theta_{t-1}}(s, a)}}   \norm{K(\theta_{t-1})^\dagger g_t  - (K(\theta_{t-1}) + \xi I)^{-1} g_t} \\
            &\overset{(i)}{\geq} - h_t C_{\xi} \sqrt{\Exp_{s, a \sim d_{*}^\rho} \sq{\norm{\psi_{\theta_{t-1}}(s, a)}^2}} \norm{g_t} \\
            &\overset{(ii)}{\geq}  -h_t C_{\xi} \sqrt{\norm{\frac{d_{*}^\rho}{d_{\theta_{t-1}}^\rho}}_\infty} \sqrt{\Exp_{s, a \sim d_{\theta_{t-1}}^\rho} \sq{\norm{\psi_{\theta_{t-1}}(s, a)}^2}}   \norm{g_t} \\
            &\overset{(iii)}{\geq}  - h_t C_{\xi} \sqrt{D_{\infty} \psi_{\infty}} \norm{g_t},
    \end{align*}
    where in $(i)$ we apply Lemma \ref{lem:diff_fisher}, and in $(ii)$ we change the measure from $d_*^\rho$ to $d_{\theta_{t-1}}^\rho$ and use the Triangle inequality, and in $(iii)$ we substitute the definition for $D_\infty$, and $\psi_{\infty}$ as defined in Assumption \ref{as:moment}. For $(a)$, we get:
    \begin{align*}
            (a): \quad \quad &\Exp_{s, a \sim d_{*}^\rho} \sq{\psi_{\theta_{t-1}}(s, a)}^\top (h_t (K_t + \xi I)^{-1} g_t - h_t (K(\theta_{t-1}) + \xi I )^{-1} g_t) \\
            &\overset{(i)}{\geq} -h_t \norm{\Exp_{s, a \sim d_{*}^\rho} \sq{\psi_{\theta_{t-1}}(s, a)}} \norm{(K_t + \xi I)^{-1} g_t - (K(\theta_{t-1})+ \xi I )^{-1} g_t} \\
            &\overset{(ii)}{\geq} -h_t\sqrt{D_{\infty} \psi_{\infty}}  \norm{(K_t + \xi I)^{-1} g_t - (K(\theta_{t-1})+ \xi I )^{-1} g_t} \\
            &\overset{(iii)}{\geq} - h_t\frac{2\sqrt{D_\infty \psi_\infty}}{\xi} \norm{g_t},
    \end{align*}
    where in $(i)$ we use the Cauchy-Schwarz inequality, in $(ii)$ we again change the measure and use Assumption \ref{as:moment}, and in $(iii)$ we use the boundedness $\norm{K_t + \xi I} \geq \xi$, $\norm{K(\theta_{t-1}) + \xi I} \geq \xi$. This term can be tightened but it will not change the order of the resultant bound.
    \begin{equation*}
        \begin{split}
            (c): \quad \quad &\Exp_{s, a \sim d_{*}^\rho} \sq{\psi_{\theta_{t-1}}(s,a)^\top h_t K(\theta_{t-1})^\dagger g_t - h_t A_{\theta_{t-1}}(s,a)} \\
            &\geq -h_t  \sqrt{\Exp_{s, a \sim d_{*}^\rho} \sq{\left(\psi_{\theta_{t-1}}(s,a)^\top K(\theta_{t-1})^\dagger g_t - A_{\theta_{t-1}}(s,a)\right)^2}}\\
            &\overset{(i)}{\geq} -h_t \sqrt{\norm{\frac{d_*^\rho}{d_{\theta_{t-1}}^\rho}}_\infty} \sqrt{\Exp_{s,a \sim d_{\theta_{t-1}}} \sq{(\psi_{\theta_{t-1}}(s,a)^\top K(\theta_{t-1})^\dagger g_t - A_{\theta_{t-1}}(s,a))^2}} \\
            &\overset{(ii)}{\geq} - h_t \sqrt{D_{\infty} E_{\Pi}},
        \end{split}
    \end{equation*}
    where in $(i)$ we substitute again the distribution mismatch coefficient, and in $(ii)$ we substitute the condition number of the policy class: $E_\Pi = \max_{\theta} \Exp_{s,a \sim d_{\theta}^\rho} \sq{\psi_{\theta}(s,a)^\top K(\theta)^\dagger \nabla J(\theta) - A_{\theta}(s,a))^2}$.
    \begin{equation*}
    \begin{split}
        (d): \quad \quad \Exp_{s, a \sim d_{*}^\rho} \sq{h_t A_{\theta_{t-1}}(s, a)} \overset{}{=} h_t (1-\gamma)(J(\theta_*) - J(\theta_{t-1})),
    \end{split}
    \end{equation*}
    by the performance difference lemma, Lemma \ref{lem:performance_diff}. Lastly we recall that the noise term is bounded by:
    \begin{equation*}
        \begin{split}
            (e): \Exp_{s, a \sim d_{*}} \sq{\psi_{\theta_{t-1}}(s, a)}^\top \Exp \sq{(h_t (K_t + \xi I)^{-1} e_t)} &\geq -\frac{h_t \sqrt{D_\infty \psi_\infty}}{\xi} \left(\frac{\sigma}{(1-\gamma) \sqrt{B}}\right). \\
        \end{split}
    \end{equation*}
    Finally, combining all of these inequalities together and taking expectation over the randomness of $\theta_t$, we get:
    \begin{align*}
        &\Exp \sq{F(\theta_t) - F(\theta_{t-1})} \\
        &\geq h_t (1-\gamma)\Exp \sq{J(\theta_*) - J(\theta_{t-1})} - C_{\nu,2} \Exp \sq{ \norm{h_t (K_t + \xi I)^{-1} (g_t + e_t)}^{\beta_2 + 1}} \\
        &\qquad - h_t \sqrt{D_\infty \psi_\infty} \bracket{\frac{\sigma}{\xi(1-\gamma) \sqrt{B}} + \frac{\sqrt{E_\Pi}}{\sqrt{\psi_{\infty}}} + \frac{C_{\xi} \xi + 2}{\xi}\Exp \sq{\norm{g_t}}} \\
        &\overset{(i)}{\geq} h_t (1-\gamma)\Exp\sq{J(\theta_*) - J(\theta_{t-1})} - \frac{C_{\nu,2} h_t^{\beta_2 + 1}}{\xi^{\beta_2+1}} \left(\Exp\sq{\norm{g_t + e_t}^{\beta_2+1}}\right) \\
        & \qquad - h_t \sqrt{D_\infty \psi_\infty} \bracket{\frac{\sigma}{\xi(1-\gamma) \sqrt{B}} + \frac{\sqrt{E_\Pi}}{\sqrt{\psi_{\infty}}} + \frac{C_{\xi} \xi + 2}{\xi}\Exp \sq{\norm{g_t}}} \\        
        &\overset{}{\geq} h_t (1-\gamma)\Exp\sq{J(\theta_*) - J(\theta_{t-1})} - \frac{2C_{\nu,2} h_t^{\beta_2 + 1}}{\xi^{\beta_2+1}} \left(\Exp\sq{\norm{e_t}^{\beta_2+1}} +\Exp \sq{\norm{g_t}^{\beta_2+1}}\right) \\
        & \qquad - h_t \sqrt{D_\infty \psi_\infty} \bracket{\frac{\sigma}{\xi(1-\gamma) \sqrt{B}} + \frac{\sqrt{E_\Pi}}{\sqrt{\psi_{\infty}}} + \frac{C_{\xi} \xi + 2}{\xi}\Exp \sq{\norm{g_t}}} \\
        &\overset{(ii)}{\geq} h_t (1-\gamma)\Exp \sq{J(\theta_*) - J(\theta_{t-1})} - \frac{2C_{\nu,2} h_t^{\beta_2 + 1}}{\xi^{\beta_2+1}} \left(\left(\frac{\sigma}{(1-\gamma) \sqrt{B}} \right)^{\beta_2+1} + \Exp \sq{\norm{g_t}^{\beta_2+1}} \right) \\
        & \qquad - h_t \sqrt{D_\infty \psi_\infty} \bracket{\frac{\sigma}{\xi(1-\gamma) \sqrt{B}} + \frac{\sqrt{E_\Pi}}{\sqrt{\psi_{\infty}}} + \frac{C_{\xi} \xi + 2}{\xi}\Exp \sq{\norm{g_t}}},
    \end{align*}
    where in $(i)$ we use that $K_t + \xi I$ has minimum eigenvalue $\xi$, in $(ii)$ we again use our bound in Lemma \ref{lem:noise}.
    
    Finally, we note that
    \begin{align*}
        F(\theta_t) - F(\theta_{t-1}) &= \Exp_{s,a \sim d_*^\rho}\sq{\log \frac{\pi_{\theta_{t-1}}}{\pi_{\theta_{t}}}(a|s)} \\
        &\leq D_{\infty} \Exp_{s,a \sim d_{\theta_{t-1}}^\rho}\sq{\log \frac{\pi_{\theta_{t-1}}}{\pi_{\theta_{t}}}(a|s)} \\
        &= D_{\infty} \Exp_{s \sim H_{\theta_{t-1}}^\rho}\sq{\Exp_{a \sim \pi_{\theta_{t-1}}(\cdot|s)} \sq{\log \frac{\pi_{\theta_{t-1}}}{\pi_{\theta_{t}}}(a|s)}} \\
        &\overset{(i)}{\leq} C_{\nu, 1} D_{\infty} \norm{\theta_{t-1} - \theta_{t}}^{\beta_1} \\
        &\leq \frac{C_{\nu,1} D_{\infty}}{\xi^{\beta_1}} h_t^{\beta_1} \left(\norm{g_t+e_t}^{\beta_1} \right) \\
        &\leq \frac{2C_{\nu,1} D_{\infty}}{\xi^{\beta_1}} h_t^{\beta_1} \left(\norm{g_t}^{\beta_1} + \norm{e_t}^{\beta_1} \right),
    \end{align*}
    where in $(i)$, we use the $\mathsf{KL}$ smoothness bound in Assumption \ref{as:smooth}.
    Thus, we find after taking expectation the following inequality:  
    \begin{equation*}
        \begin{split}
            J(\theta_*) - \Exp \sq{J(\theta_{t-1})} &\leq \frac{2C_{\nu,1} D_{\infty}}{\xi^{\beta_1}(1-\gamma)} h_t^{\beta_1-1} \left(\left(\frac{\sigma}{\xi(1-\gamma) \sqrt{B}}\right)^{\beta_1}  + \Exp \sq{\norm{g_t}^{\beta_1}} \right) \\
            &\qquad + \frac{2C_{\nu,2} h_t^{\beta_2}}{\xi^{\beta_2+1}(1-\gamma)} \left(\left(\frac{\sigma}{\xi(1-\gamma) \sqrt{B}} \right)^{\beta_2+1} + \Exp \sq{\norm{g_t}^{\beta_2+1}} \right) \\
            &\qquad + \frac{\sqrt{D_\infty \psi_\infty}}{1-\gamma} \bracket{\frac{\sigma}{\xi(1-\gamma) \sqrt{B}} + \frac{\sqrt{E_\Pi}}{\sqrt{\psi_{\infty}}} + \frac{C_{\xi} \xi + 2}{\xi}\Exp \sq{\norm{g_t}}}. 
        \end{split}
    \end{equation*}
    This concludes the proof if we write $C_{NPG, 3} = \frac{2C_{\nu,1} D_{\infty}}{\xi^{\beta_1}}$, $C_{NPG, 4} =  \frac{2C_{\nu,2} h_t^{\beta_2 + 1}}{\xi^{\beta_2+1}}$ and $C_{NPG, 5} = \frac{(C_{\xi}\xi +2)}{\xi}\sqrt{D_{\infty} \psi_{\infty}}.$
    \hfill $\square$
    
    In the following, we take the minimum of the right hand side over $t=1 \ldots T$.

    \textit{Proof of Corollary \ref{cor:opt}, NPG}:
    
    It remains now to substitute the learning rate and gradient bounds from each of the cases we analyzed earlier. Since $\beta_1 \geq 1$, $\beta_2 > 0$, the final term always dominates the order when we consider the dependencies in $\epsilon, (1-\gamma)$. In the sequel, we will ignore the leading terms, noting that they may contribute to other dependencies in $\xi$, $D_{\infty}$, etc.

     \textbf{Case I, $h_t = \lambda$}: Recall that in this case, 
    \begin{align*}
    \min_{t \leq T} \norm{g_t}^2 &\leq\frac{4(\psi_\infty+\xi)}{\lambda T} \left( J_* - J(\theta_0)\right)+ \left(\frac{\sigma}{(1-\gamma) \sqrt{B}}\right)^{\holde+1}\\
    &\qquad + 2^{\frac{8}{1-\holde}} \left(\frac{\left(L_1+L_2\right) (\psi_\infty+\xi)}{(1-\gamma)\xi^{\maxholde+1}}\right)^{\frac{2}{1-\holde}} \lambda^{\frac{2\holde}{1-\holde}}.
    \end{align*}
    Subsequently, we can choose $\lambda \lesssim (1-\gamma)^{\frac{2-\holde}{\holde}} \epsilon^{\frac{1-\holde}{\holde}}$, so that the final term is $\lesssim \epsilon^2 (1-\gamma)^2$.
    
    Subsequently we require $T \gtrsim \frac{1}{\lambda \epsilon^2 (1-\gamma)^2}$ to make the first term $\lesssim (1-\gamma)^2\epsilon^2$  and $B \gtrsim \epsilon^{-\frac{4}{\holde+1}} (1-\gamma)^{-2 - \frac{4}{\holde+1}}$ to make the second term $\lesssim (1-\gamma)^2\epsilon^2$. 
   Substituting our requirement on $\lambda$, we get $T \gtrsim \epsilon^{-\frac{1+\holde}{\holde}} (1-\gamma)^{-\frac{2+\holde}{\holde}}.$ Thus, we get $\min_{t \leq T} \norm{g_t}^2 \lesssim (1-\gamma)^2\epsilon^2$. This then implies
    \begin{align*}
        \min_{t\leq T} J_* - \Exp \sq{J(\theta_t)} \lesssim \epsilon + \frac{\sqrt{D_\infty E_\Pi}}{1-\gamma}.
    \end{align*}
    
 \textbf{Case II, $h_t = \lambda T^{\frac{\holde-1}{\holde+1}}$}:  Recall that in this case, 
    $$\min_{t \leq T} \norm{g_t}^2 \leq \frac{4(\psi_\infty+\xi)}{\lambda} T^{-\frac{2\holde}{\holde+1}} \left(J_* - J(\theta_0) + \frac{\lambda}{4} \right) + T^{\frac{\holde(\holde-1)}{\holde+1}} \left(\frac{\sigma}{(1-\gamma) \sqrt{B}}\right)^{\holde+1}.$$
    Subsequently we require $T \gtrsim \left(\frac{1}{\lambda \epsilon^2 (1-\gamma)^2}\right)^{\frac{\holde+1}{2\holde}}$ to bound the first term by $\lesssim \epsilon^{2} (1-\gamma)^{2}$. For simplicity, take $\lambda \asymp (1-\gamma)^{1/\holde}$, in which case we get $T \gtrsim \left(\frac{1}{\epsilon^2 (1-\gamma)^{2+{1/\holde}}}\right)^{\frac{\holde+1}{2\holde}}$.
    Consequently, the second term is bounded by
    \begin{align*}
        T^{\frac{\holde(\holde-1)}{\holde+1}} \left(\frac{\sigma}{(1-\gamma) \sqrt{B}}\right)^{\holde+1} &\lesssim \left(\frac{1}{\epsilon^2 (1-\gamma)^{2+1/\holde}}\right)^{\frac{\holde-1}{2}} \left(\frac{1}{(1-\gamma) \sqrt{B}}\right)^{\holde+1} \\
        &\lesssim \epsilon^{1-\holde} (1-\gamma)^{-\frac{4\holde^2+\holde-1}{2\holde}}B^{-\frac{\holde+1}{2}}.
    \end{align*}
    So we can take $B \gtrsim \epsilon^{-2} (1-\gamma)^{\frac{4\holde^2 + 5\holde - 1}{\holde(\holde+1)}}$,  to make this term $\lesssim (1-\gamma)^2\epsilon^2$. This then implies
    \begin{align*} 
        \min_{t\leq T} J_* - \Exp \sq{J(\theta_{t-1})} \leq \epsilon + \frac{\sqrt{D_\infty E_{\Pi}}}{1-\gamma}.
    \end{align*}
    
    \hfill $\square$

    \textbf{Remark:} It is possible to do similar analysis for policy gradient using this technique, but the quantity $E_{\Pi}$ lacks a meaningful interpretation in that case.

\section{Experiments}
All results of these experiments can be found in the two Figures in the main text, and they are primarily intended to demonstrate practical examples of policies that satisfy our assumptions. All experiments were run locally on a single CPU, with 1Gb of dedicated RAM. The experiments are not computationally intensive and are primarily illustrative.

\subsection{Generalized Gaussian Policies}
For Figure \ref{fig:generalized_gaussian}(a), we sampled $4000$ actions from a generalized Gaussian with mean $\inner{s}{\theta}$ with $\theta \in \R^2$ and parameter $\kappa = 1.2$, for a state $[1.0,0]$ in $\R^2$. We let the first coordinate $\theta^{(1)}$ range from $[-2, 2]$ and compute the score function is then computed for these actions. The magnitude of the difference (compared to the reference point, $\theta^{(1)} = 0$) is then generated, and plotted against the standard Gaussian ($\kappa = 2$). The tail growth properties as well as local smoothness can be seen clearly from the Figure.

For the exploration problem, we implement the reward function described in Equation \eqref{eq:example} with parameter $\theta^* = 3.9$ and initial parameter $\theta_0 = 0$. We aggregate gradients using a batch size of $1000$ and update using our policy gradient rule. The performance is then reported on a batch of $1000$, with standard deviations reported within this batch. We ran this experiment for 5 seeds and found that the qualitative behaviour does not change (although the inflection point of the policy behaviour is highly random, the Gaussian policy is always significantly worse than the generalized Gaussian).

\subsection{Unbounded Gradient}
We devise a simple policy and show that its gradient appears to be unbounded (in fact scaling logarithmically with the number of samples), but has a bounded integral according to its own measure $\pi_\theta$. For this, we simply implement the policy given in Example \ref{ex:l2}, and sample larger batches of actions. We then compute the gradients, and compare two norms that have been used in assumptions, that is the norm $\max_{n=1, \ldots K} \norm{\psi_\theta(s,a_n)}$ and our assumption $\frac{1}{K} \sum_{n=1}^K \norm{\psi_\theta(s, a_n)}^2$. We find that the former is unbounded for larger sets while the latter converges smoothly to a fixed value. As we explain in our paper, this shows that interesting policies can be covered by our assumptions while not being allowed in prior work.

\end{document}